\documentclass{article} % For LaTeX2e
\usepackage{iclr2025_conference_arxiv,times}
\usepackage{minitoc}
\usepackage[nottoc]{tocbibind}
\usepackage[toc,page,header]{appendix}
\usepackage{amsmath,amsfonts,bm}

\newcommand{\figleft}{{\em (Left)}}

\newcommand{\figright}{{\em (Right)}}
\newcommand{\figtop}{{\em (Top)}}
\newcommand{\figmiddle}{{\em (Middle)}}
\newcommand{\figbottom}{{\em (Bottom)}}

\def\eqref#1{equation~\ref{#1}}
\def\1{\bm{1}}

\DeclareMathAlphabet{\mathsfit}{\encodingdefault}{\sfdefault}{m}{sl}
\SetMathAlphabet{\mathsfit}{bold}{\encodingdefault}{\sfdefault}{bx}{n}

\DeclareMathOperator*{\argmax}{arg\,max}
\DeclareMathOperator*{\argmin}{arg\,min}

\usepackage{graphicx}
\usepackage{hyperref}
\usepackage{url}
\usepackage{wrapfig}
\hypersetup{
  colorlinks,
  citecolor=green,
  linkcolor=violet,
  urlcolor=blue
}
\usepackage{floatrow} 

\usepackage{adjustbox}
\usepackage{algorithm}
\usepackage{algpseudocode}

\usepackage{stmaryrd}
\usepackage{trimclip}
\makeatletter
\DeclareRobustCommand{\shortto}{%
  \mathrel{\mathpalette\short@to\relax}%
}
\newcommand{\short@to}[2]{%
  \mkern2mu
  \clipbox{{.5\width} 0 0 0}{$\m@th#1\vphantom{+}{\shortrightarrow}$}%
  }
\makeatother

\usepackage{tikz}
\usetikzlibrary{arrows.meta, positioning}

\usepackage{amsmath}
\usepackage{amssymb}
\usepackage{mathtools}
\usepackage{amsthm}
\usepackage{bbm}

\theoremstyle{plain}
\newtheorem{theorem}{Theorem}[section]
\newtheorem{proposition}[theorem]{Proposition}

\theoremstyle{definition}

\theoremstyle{remark}

\usepackage{cancel}

\title{
    \textit{Jump Your Steps}: Optimizing Sampling Schedule of Discrete Diffusion Models
}

\author{Yonghyun Park\thanks{Work done during an internship at SONY AI.}~~\&~Chieh-Hsin Lai \\
Sony AI \\
Tokyo, Japan \\
\texttt{enkeejunior1@snu.ac.kr} \\
\And
Satoshi Hayakawa \\
Sony Group Corporation \\
Tokyo, Japan \\
\And
Yuhta Takida \\
Sony AI \\
Tokyo, Japan \\
\And
Yuki Mitsufuji \\
Sony AI~\&~Sony Group Corporation \\
New York, USA \\
}

\usepackage{amsmath,amsfonts,bm}
\usepackage{amsbsy}

\usepackage{graphicx}

\newcommand{\fref}[1]{Figure~\ref{#1}}

\newcommand{\eref}[1]{Eq.~(\ref{#1})}

\newcommand{\sref}[1]{Section~\ref{#1}}
\newcommand{\aref}[1]{Appendix \ref{#1}}

\newcommand{\yh}[1]{{\color{black}{{#1}}}}

\iclrfinalcopy % Uncomment for camera-ready version, but NOT for submission.
\begin{document}
\doparttoc % Tell to minitoc to generate a toc for the parts
\faketableofcontents % Run a fake tableofcontents command for the partocs

\maketitle

\begin{abstract}

Diffusion models have seen notable success in continuous domains, leading to the development of discrete diffusion models (DDMs) for discrete variables. Despite recent advances, DDMs face the challenge of slow sampling speeds. While parallel sampling methods like $\tau$-leaping accelerate this process, they introduce \textit{Compounding Decoding Error} (CDE), where discrepancies arise between the true distribution and the approximation from parallel token generation, leading to degraded sample quality. In this work, we present \textit{Jump Your Steps} (JYS), a novel approach that optimizes the allocation of discrete sampling timesteps by minimizing CDE without extra computational cost. More precisely, we derive a practical upper bound on CDE and propose an efficient algorithm for searching for the optimal sampling schedule. Extensive experiments across image, music, and text generation show that JYS significantly improves sampling quality, establishing it as a versatile framework for enhancing DDM performance for fast sampling.

\end{abstract}

\section{Introduction}

{
Diffusion models~\citep{sohl2015deep, song2020score, ho2020denoising, song2020denoising, karras2022elucidating} have achieved remarkable success in generation tasks within the continuous domain. However, certain modalities, such as text and music, inherently possess discrete features. Recently, discrete diffusion models (DDMs)~\citep{austin2021structured, campbell2022continuous, campbell2024generative, gat2024discrete} have demonstrated performance comparable to state-of-the-art methods in various areas, including text~\citep{lou2023discrete, shi2024simplified} and image~\citep{chang2022maskgit, gu2022vector} generation. Nevertheless, like their continuous counterparts, DDMs encounter a significant bottleneck in sampling speed due to their progressive refinement process.

In contrast to continuous-domain diffusion models, where sampling dynamics are driven by sample-wise differential equations~\citep{song2020score}, allowing for the direct application of well-established numerical methods to accelerate generation, enhancing speed in DDMs poses a significant challenge. To address this, researchers have proposed fast and efficient samplers, including notable methods such as the $\tau$-leaping~\citep{campbell2022continuous, lezama2022discrete, sun2022score} and $k$-Gillespie algorithms~\citep{zhao2024informed}, which facilitate parallel sampling of multiple tokens in a single step. However, this parallel but independent sampling introduces \textit{Compounding Decoding Error} (CDE)~\citep{lezama2022discrete}, which arises from a mismatch between the training and inference distributions of intermediate latents during parallel sampling. Specifically, while each token is generated according to its marginal distribution, the joint distribution deviates from the learned distribution. To mitigate this issue, the predictor-corrector (PC) sampler~\citep{campbell2022continuous} has been proposed. This sampler slightly perturbs the generated data to correct incorrectly generated tokens. However, these methods have limitations, including impracticality under low computational budgets~\citep{campbell2022continuous}, the need for an additional corrector~\citep{lezama2022discrete}, or reliance on specialized architectures and loss functions~\citep{zhao2024informed}.

\begin{figure}[!h]
    \centering
    \includegraphics[width=1\linewidth]{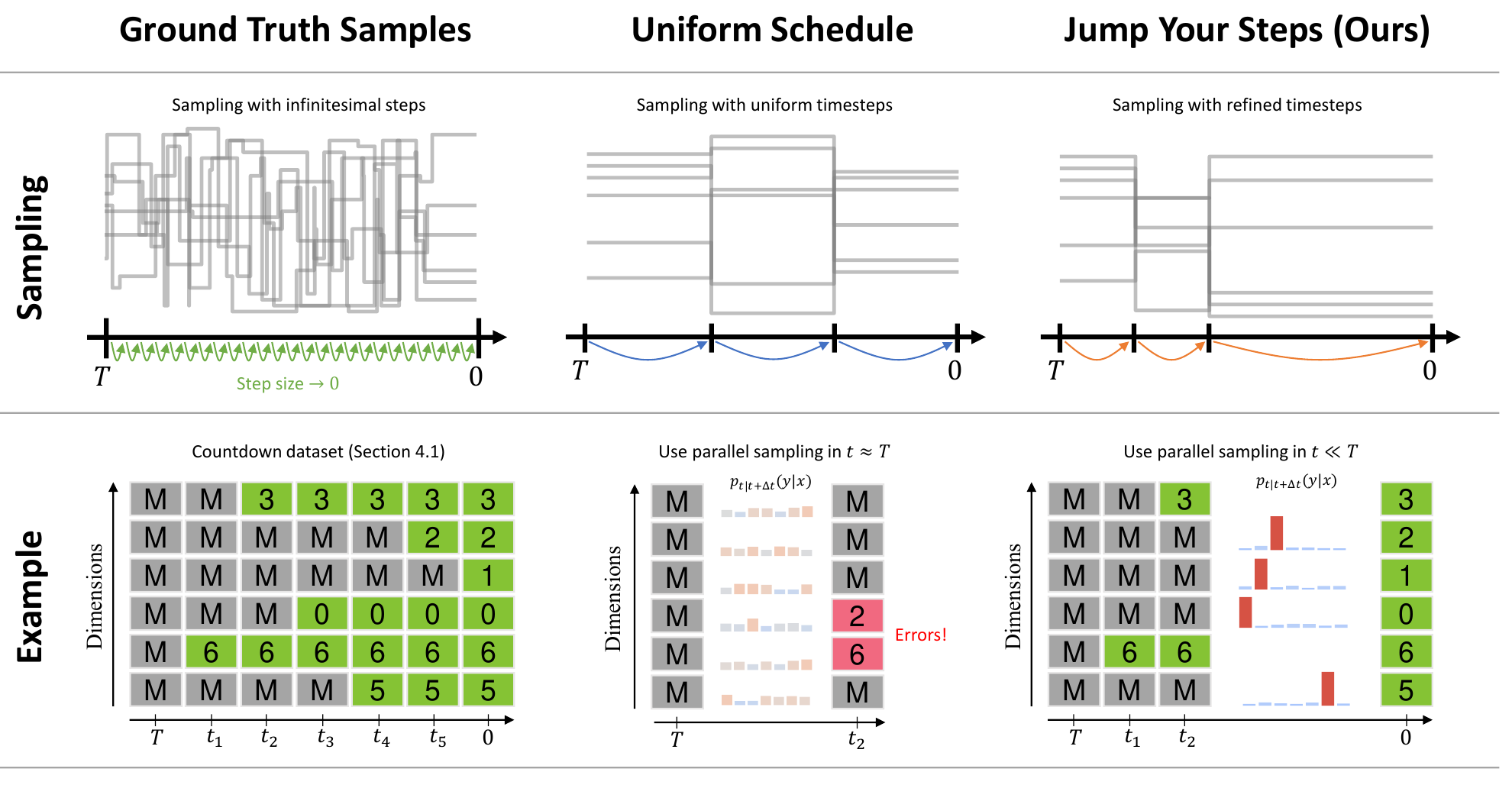}
    \caption{
    \figtop{}~Comparison of sampling trajectories: ground truth vs. parallel sampling using a uniform schedule and the Jump Your Steps (JYS) schedule. \figbottom{}~\yh{Uniform schedule exhibits compounding decoding errors during parallel sampling, while JYS reduces them by using fewer steps in deterministic phases and reallocating skipped steps to other timesteps.}
    }
    \label{fig:teaser}
\end{figure}

To reduce CDE and enable fast sampling in DDM fundamentally, we first introduce a rigorous quantity to measure CDE (see \fref{fig:teaser} Top, and \sref{subsection:section3.1}) and propose a novel approach called \textit{Jump Your Steps} (JYS), which optimizes the allocation of discrete sampling timesteps \(\{ T \shortto t_1  \shortto \dots \shortto t_{N-1} \shortto 0 \}\) under a fixed total sampling budget $N$ to minimize CDE. Our core idea is to derive efficiently computable bounds for CDE (see \sref{subsection:section.add} and \sref{subsection:section3.2}) and strategically select sampling timesteps by solving minimization problems to reduce these bounds (\sref{subsection:section3.3} and \sref{subsection:section3.4}), theoretically ensuring a decrease in the gap between the ground truth distribution and the approximated distribution through parallel sampling (see \fref{fig:theory-relationship}).

Unlike previous methods such as the PC sampler, our approach requires no additional computational resources or modifications to the model architecture or loss function. We empirically validate the effectiveness of our sampling schedule across various datasets, including synthetic sequential data, CIFAR-10 (image), Lakh Pianoroll (music), and text modeling. Our approach accelerates DDM sampling across models using different forward corruption transition kernels, such as uniform, Gaussian, and absorbing transition matrices. Our comprehensive experiments cover both unconditional and conditional generation tasks, consistently showing that optimizing the sampling schedule significantly enhances sampling quality. These results indicate that our method serves as a general framework for speeding up discrete diffusion model sampling.}

% We introduce a principled approach to estimate the error induced by parallel sampling (Section 3.2) and propose an efficient algorithm to optimize the sampling schedule for minimizing the error given the same compute budget (Section 3.3). Specifically, we derive an upper bound on the compounding decoding error using theoretical tools such as Girsanov's Theorem. This allows us to quantify the error introduced at each timestep and formulate an optimization problem to find the sampling schedule that minimizes the total compounding decoding error under a fixed computational budget. We propose an efficient algorithm to solve this optimization problem, enabling practical implementation of our method within a feasible amount of computation.

\section{Background}

\subsection{Continuous Time Framework for Discrete Diffusion Models.}

{
% --------------------------------

Discrete diffusion models (DDMs) define the generative process as the reverse of the data-corrupting forward process, expressed as a Continuous Time Markov Chain (CTMC) on a finite state space $\mathcal{S}$~\citep{campbell2022continuous}. For the data-corrupting process $(X_t)_{t\in [0,T]}$, the density evolution is described as: 
\begin{equation}
q_{t+dt|t}(y \mid x) = \delta_{x y} +R_t(x,y) dt + o(dt)
\end{equation}
Here, $\delta_{x y}$ is the Dirac delta function, $R_t \in \mathbb{R}^{S\times S}$ is the transition rate matrix of the forward CTMC, with $S=|\mathcal{S}|$, and $dt > 0$. Rate matrices ensure the marginal distribution $q_t(x_t) = \int q_t(x_t|x_0)q_0(x_0) dx_0$, where $q_0 = p_{\mathrm{data}}$ and $q_T \approx \pi$, the stationary distribution of the forward CTMC. Various transition matrices have been proposed, allowing $\pi$ to follow a uniform or Gaussian distribution, or converting samples into masked tokens.

For generation, we reverse the forward process, moving from the marginal $q_T$ back to $p_{\mathrm{data}}$. This time-reversal CTMC is also a CTMC~\citep{campbell2022continuous, campbell2024generative}: 
\begin{equation} \label{eq:reverse_kolmo}
q_{t-dt|t}(y \mid x) = \delta_{x y} +\tilde{R}(x, y) dt + o(dt), \end{equation} 
where the backward transition rate $\tilde{R}$ is defined as: 
\begin{equation*}
\tilde{R}(x, y) 
= R(y, x)\underbrace{\frac{q_t(y)}{q_t(x)}}_{\mathclap{\textrm{Score Parametrization}}} 
= R(y, x)\sum_{x_0} \frac{q_t(y \mid x_0)}{q_t(x\mid x_0)} \underbrace{q_{0|t}(x_0 \mid x)}_{\mathclap{\textrm{Denoising Parametrization}}} 
\end{equation*}
The literature primarily falls into two parameterizations: Denoising parameterization~\citep{campbell2024generative, austin2021structured, campbell2022continuous} approximates a parameterized denoising model as $p^\theta_{0|t}(x_0|x) \approx q_{0|t}(x_0 | x)$. Conversely, score parameterization~\citep{lou2023discrete, meng2022concrete} estimates the ratio of the data distribution as $s_t^\theta(y|x)={q_t(y)}/{q_t(x)}$.}

\subsection{Sampling from the backward CTMC}
\label{section:background}

\yh{
\paragraph{Gillespie’s Algorithm} was proposed as a simulation algorithm for a given CTMC \citep{gillespie2007stochastic}. Gillespie's algorithm simulates the CTMC by calculating the rate matrix at each state transition. If the rate matrix of the CTMC depends only on the state, Gillespie’s algorithm serves as an exact simulation method. However, since it allows for only one token transition each time the rate matrix is calculated, it is computationally inefficient.

\paragraph{$k$-Gillespie’s Algorithm} Instead of updating only one token for each rate matrix calculation, the $k$-Gillespie's algorithm \citep{zhao2024informed} updates \( k \) tokens in parallel. This reduces the computation by a factor of \( 1/k \) compared to the original Gillespie algorithm.

\paragraph{$\tau$-Leaping} On the other hand, \cite{campbell2022continuous} proposes sampling through $\tau$-leaping. Unlike the $k$-Gillespie algorithm, which update $k$ tokens in parallel, $\tau$-leaping simultaneously updates all tokens according to the given fixed rate matrix within the specified time interval $[t, t+\tau)$. Recently, Tweedie $\tau$-leaping, which considers changes in the rate matrix according to the noise schedule, has been proposed \citep{sun2022score, lou2023discrete}. 
}

\section{Optimizing the sampling schedule of discrete diffusion models}

% For a fixed budget of sampling steps $n$, i.e., Number of Function Evaluation (NFE), our goal is to find an optimal sampling timestep \((T, t_1, \dots, t_n, 0)\).

In this section, we aim to optimize sampling schedule \(\{ T \shortto t_1 \shortto t_2 \shortto \dots \shortto t_{N-1} \shortto 0 \}\) to minimize the \textit{Compounding Decoding Error} (CDE) introduced by parallel sampling. First, we define and analyze the CDE, examining its relationship to both the sampling schedule (\sref{subsection:section3.1}) and sampling quality (\sref{subsection:section.add}). In \sref{subsection:section3.2}, we derive an upper bound on the CDE, which serves as the objective for the sampling schedule optimization. Finally, we introduce a hierarchical breakdown strategy (\sref{subsection:section3.3}) and computational techniques (\sref{subsection:section3.4}) to make the optimization tractable. 
\fref{fig:theory-relationship} summarizes the relationships between the theoretical analyses discussed in this section.

Although this section focuses on samplers based on \(\tau\)-leaping, all methods are also applicable to the \(k\)-Gillespie algorithm. For extensions to \(k\)-Gillespie, please refer to the Algorithm~\ref{alg:klub}.

\paragraph{Notations} 
To begin, we introduce some essential mathematical notation. \( X : \) a random variable, \( \mathbf{x} : \) its observation, \( \mathbb{P}, \mathbb{Q} : \) distributions, \(\{ T \shortto t_1 \shortto \dots \shortto t_{N-1} \shortto 0 \} :\) sampling schedule, and \(\mathbb{Q}^{a \shortto b \shortto \cdots \shortto c} :\) the distribution generated by the sampling schedule \(\{ a \shortto b \shortto \cdots \shortto c \}\). For clarity, when working with backward CTMCs, we slightly abuse notation and express intervals as \([s, t] \triangleq \{u \mid s \ge u \ge t\}\);
the same applies to open and half-open intervals.
% To begin, we introduce some essential mathematical notation. Let \( X \) represent a random variable and \( \mathbf{x} \) its observation, with \( \mathbb{P} \) and \( \mathbb{Q} \) denoting true distributions, and approximated distribution, respectively. The sampling schedule is given by \(\{ T \shortto t_1 \shortto \dots \shortto t_{N-1} \shortto 0 \}\), and \(\mathbb{Q}^{a \shortto b \shortto \cdots \shortto c}\) represents the distribution generated by the sampling schedule \(\{ a \shortto b \shortto \cdots \shortto c \}\). For clarity, when working with backward CTMCs, we slightly abuse notation and express intervals as \([s, t] \triangleq \{u \mid s \ge u \ge t\}\).

\begin{figure}[th]
    \centering
    \usetikzlibrary{calc} 
    \begin{tikzpicture}[
        node distance=3cm and 3.5cm, 
        every node/.style={align=center, font=\footnotesize},
        block/.style={rectangle, draw, minimum height=2em, text centered, font=\footnotesize},
        root/.style={rectangle, draw, thick, font=\footnotesize, text centered, minimum height=2em},
        dashedblock/.style={rectangle, draw, dashed, thick, minimum height=2em, font=\footnotesize}
    ]
    
    % Nodes for inequalities
    \node (DKL_0) {$\mathcal D_{\mathrm{KL}}(\mathbb{P}_0 \| \mathbb{Q}_0^{t_0\shortto t_1 \shortto \cdots \shortto 0})$ }; % \jcc{should the $\mathbb{Q}_0$  be discretized chain: $$?}
    \node[right=0.1cm of DKL_0] (leq1) {$\leq$};
    \node[right=0.1cm of leq1] (DKL_t) {$\sum_{i=0}^{N-1} \mathcal D_{\mathrm{KL}}(\mathbb{P}_{t_{i+1}} \| \mathbb{Q}^{t_{i}\shortto t_{i+1}}_{t_{i+1}})$};
    \node[right=0.1cm of DKL_t] (leq2) {$\leq$};
    \node[right=0.1cm of leq2] (DKL_path) {$\mathcal D_{\mathrm{KL}}(\mathbb{P}_{\mathrm{paths}} \| \mathbb{Q}_{\mathrm{paths}}^{t_0\shortto t_1 \shortto \cdots \shortto 0})$};

    % Labels for Theorems
    \node[above=0.025cm of leq1] (thm_3_1) {Theorem \ref{theorem:1}};
    \node[above=0.025cm of leq2] (eq6) {\eref{equation:CDE-KLUB}};

    % Vertical branches
    \node[below=0.75cm of DKL_t] (CDE) {$\sum_{i=0}^{N-1} \mathcal{E}_{\mathrm{CDE}}(t_{i} \shortrightarrow t_{i+1})$};
    \node[below=0.75cm of DKL_path] (KLUB) {KLUB$(\mathbb{P}_{0} \| \mathbb{Q}_{0}^{t_0\shortto t_1 \shortto \cdots \shortto 0})$};
    \node[below=0.75cm of KLUB] (JYS) {$\{ T \shortto t_1 \shortto \dots \shortto t_{N-1} \shortto 0 \}$};

    % Arrows for the flow
    \draw[dashed] (DKL_t) -- (CDE) node[midway, right] {Eqs.~(\ref{equation:CDE}, \ref{equation:CDE_marginal})};
    \draw[dashed] (DKL_path) -- (KLUB) node[midway, right] {Theorem \ref{theorem:2}};
    \draw[->] (KLUB) -- (JYS) node[midway, right] {Algorithm \ref{alg:klub}, \ref{alg:jump-your-steps}.};
    
    \end{tikzpicture}
    \caption{
    An illustration of the relationship between the KL divergence of the distribution, the compounding error \(\mathcal{E}_{\mathrm{CDE}}\) (\sref{subsection:section3.1}), and KLUB (\sref{subsection:section3.2}). The sampling schedule \(\{ T \shortto t_1 \shortto \dots \shortto t_{N-1} \shortto 0 \}\) is optimized to minimize KLUB using the efficient algorithms detailed in Section~\ref{subsection:section3.3}, and \ref{subsection:section3.4}.
    % An illustration of the relationship between the KL divergence of the distribution, the compounding error \(\mathcal{E}_{CDE}\) (\sref{subsection:section3.1}), and KLUB (\sref{subsection:section3.2}). We optimize sampling schedule \((T, t_1, ..., 0)\) to minimize KLUB using efficient algorithms discussed in Section~\ref{subsection:section3.3}, \ref{subsection:section3.4}.
    }
\label{fig:theory-relationship}
\end{figure}

\subsection{
Time-Dependent Nature of Compounding Decoding Errors}
\label{subsection:section3.1}

We introduce a measure for the \emph{Compounding Decoding Error} (CDE), $\mathcal E_{\mathrm{CDE}}$, which quantifies the discrepancy between the true joint distribution and the distribution from parallel token generation. For illustration, we consider a discrete process $X_t = (X^1_t, X^2_t)$ with sequence length 2, consisting of tokens \( X^1_t \) and \( X^2_t \). The general case is provided in \aref{appendix:theorem1}.

We propose measuring the CDE for a single parallel sampling step $\{s \shortto t\}$ start from $\mathbf{x}_s$ by using the KL divergence between the joint distribution \( P_{X^1_t,X^2_t|\mathbf{x}_s} \) and the product of marginal distributions \( P_{X^1_t|\mathbf{x}_s} \otimes P_{X^2_t|\mathbf{x}_s} \):

\begin{equation}
\mathcal E_{\mathrm{CDE}}(s \shortrightarrow t | \mathbf{x}_s) 
\triangleq 
\mathcal D_{\mathrm{KL}}(
\underbrace{P_{X^1_t, X^2_t|\mathbf{x}_s}}_{\textrm{True distribution}} 
\| 
\underbrace{P_{X^1_t|\mathbf{x}_s} \otimes P_{X^2_t|\mathbf{x}_s}  }_{\textrm{Approx.\ distribution from parallel sampling}} 
).
\label{equation:CDE}
\end{equation}
We note that the defined CDE is equivalent to the conditional mutual information \( \mathcal I(X^1_t; X^2_t | \mathbf{x}_s) \) of tokens \( X^1_t, X^2_t\):
\begin{equation}
\mathcal E_{\mathrm{CDE}}(s \shortrightarrow t | \mathbf{x}_s) = \mathcal I(X^1_t; X^2_t  | \mathbf{x}_s).
\label{equation:mi}
\end{equation}

This expression links the compounding error to the mutual information between tokens; lower mutual information reduces parallel sampling errors. For example, as shown in \fref{fig:teaser} \figbottom{}, as generation progresses, the uncertainty of each token decreases over time due to the tokens already generated, reducing mutual information and preventing CDE. In general, CDE depends on the timesteps, and its behavior varies with the data distribution, corruption kernel (see Fig.~\ref{figure:jump-your-steps} for illustration), and DDM sampling methods. \yh{Motivated by this observation, we hypothesize that we can reduce the CDEs during generation process by optimizing the sampling schedule.} % $(T, t_1, t_2, \dots, t_{N-1}, 0)$} % \( T = t_0 > t_1 > t_2 > \dots > t_{N-1} > t_N = 0 \).
% \jcc{this paragraph should be moved to after the definition in Eq. (3)}

\subsection{Relation between Compounding Decoding errors and generation quality}
\label{subsection:section.add}

While the \eref{equation:CDE} allows us to estimate the CDE starting from a specific state \(\mathbf{x}_s\), in practice, we are interested in the average compounding error over all possible starting states at time \( s \). To assess the overall impact of the CDE when transitioning over the timesteps \( s \shortrightarrow t \), we consider the expected value of \(\mathcal{E}_{\mathrm{CDE}}(s \shortrightarrow t | \mathbf{x}_s)\) with respect to \(\mathbf{x}_s \sim \mathbb{P}_s\). This leads us to consider:
\begin{equation}
\mathcal E_{\mathrm{CDE}}(s \shortrightarrow t) 
\triangleq \mathbb{E}_{\mathbf{x}_s} \left[ \mathcal E_{\mathrm{CDE}}(s \shortrightarrow t | \mathbf{x}_s) \right].
\label{equation:CDE_marginal}
\end{equation}

Consider a sampling schedule \( \{ T = t_0 \shortto t_1 \shortto \dots \shortto t_{N-1} \shortto 0 = t_{N} \} \), which will be specified later. Our goal is to minimize the cumulative CDE that arises from each parallel sampling step within the given schedule. If we ignore the accumulated error from the previous steps that affects the consecutive steps, our objective is as follows:
\begin{equation}
    \min_{t_1, t_2, \dots, t_{N-1}} \sum_{i=0}^{N-1} \mathcal{E}_{\mathrm{CDE}}(t_i \shortrightarrow t_{i+1}).
\end{equation}

% Interestingly, Theorem 3.1 shows that reducing the cumulative CDE also reduces the KL divergence between the true distribution at time \( t = 0 \), denoted \( \mathbb{P}_0 \), and the distribution \( \mathbb{Q}_0^{T \shortto t_1 \shortrightarrow \cdots \shortto 0} \) obtained from parallel sampling along the sampling schedule (see \aref{appendix:theorem1} for proof).
Interestingly, we find in the following theorem that cumulative CDEs over the sampling schedule can upper bound the KL divergence between the true distribution at time \( t = 0 \), denoted \( \mathbb{P}_0 \), and the distribution \( \mathbb{Q}_0^{T\shortto t_1 \shortto \cdots \shortto 0} \) obtained from parallel sampling along the sampling schedule (see \aref{appendix:theorem1} for proof):

% Interestingly, there are close relationship between cumulative CDEs and the KL divergence between the true distribution at time \( t = 0 \), denoted \( \mathbb{P}_0 \), and the distribution $\mathbb{Q}_0^{T\shortto t_1 \shortto \cdots \shortto 0}$ obtained from parallel sampling along the sampling schedule (see \aref{appendix:theorem1} for proof):

% We present a theorem that relates the CDE to the KL divergence between the true distribution at time \( t = 0 \), denoted \( \mathbb{P}_0 \), and the distribution $\mathbb{Q}_0^{T\shortto t_1 \shortto \cdots \shortto 0}$ obtained from parallel sampling along the sampling schedule (see \aref{appendix:theorem1} for proof):
\begin{theorem} 
\label{theorem:1}
We have the following bound on the KL divergence between \( \mathbb{P}_0 \) and \( \mathbb{Q}_0^{T \shortto t_1 \shortto \cdots \shortto 0} \) in terms of cumulative CDEs:
\begin{equation}\label{eq:KL_CDE}
    \mathcal{D}_{\mathrm{KL}}(\mathbb{P}_0 \| \mathbb{Q}_0^{T\shortto t_1 \shortto \cdots \shortto 0}) \le \sum_{i=0}^{N-1} \mathcal{E}_{\mathrm{CDE}}(t_i \shortrightarrow t_{i+1}).
\end{equation}
\end{theorem}
{This theorem suggests that effectively allocating the time schedule to minimize the CDEs can implicitly reduce the discrepancy between the true distribution and the approximate distribution obtained from parallel sampling. Motivated by this, in the next section, we derive a tractable upper bound for 
\( \sum_{i=0}^{N-1} \mathcal{E}_{\mathrm{CDE}}(t_i \shortrightarrow t_{i+1}) \) that depends on the sampling schedule \( \{  T \shortto t_1 \shortto \dots \shortto t_{N-1} \shortto 0 \} \) to facilitate its optimization.}
% \( \mathcal{E}_{\mathrm{CDE}}(s \shortrightarrow t) \) that depends on the timestep interval \( s \shortrightarrow t \) to facilitate the search for the sampling schedule.}

% {Motivated by this observation, we seek a sequence of sampling times \( T = t_0 > t_1 > t_2 > \dots > t_{N-1} > t_N = 0 \) that reduces the \yh{cumulative} CDEs with respect to the sampling schedule, thereby minimizing the gap between the approximated distribution \( \mathbb{Q}_0^{T \shortto t_1 \shortto \cdots \shortto 0} \) and the true data distribution \( \mathbb{P}_0 \).}

\subsection{Estimating the Compounding Decoding Error Using Girsanov's Theorem}
\label{subsection:section3.2}

As shown in \eref{equation:CDE}, computing $\mathcal E_{\mathrm{CDE}}(s \shortrightarrow t | \mathbf{x}_s)$ involves determining the KL divergence between the true distribution and the approximated distribution from DDM's parallel sampler, which is often intractable. To address this, we treat the ground truth reverse process and the sampling process from DDM's parallel samplers (introduced in \sref{section:background}) as two CTMCs, starting from the same initial distribution. By applying \textbf{Girsanov's theorem}~\citep{ding2021markov, chen2022sampling}, we derive a tractable formula to compare the KL divergence between the distributions of these stochastic processes at any time interval $[s, t]$. We summarize this as the following general theorem applicable to any two backward CTMCs with $R^1_t$ and $R^2_t$ as their respective transition rate matrices:
\[
\begin{cases}
\textrm{CTMC 1}: & q^1_{u - du|u}(y \mid x) = \delta_{xy} + R^1_t(x,y) du + o(du), \\ 
\textrm{CTMC 2}: & q^2_{u - du|u}(y \mid x) = \delta_{xy} + R^2_t(x,y) du + o(du).
\end{cases}
\]
 We defer its proof to \aref{appendix:theorem2}.

\begin{theorem}
\label{theorem:2}
{\normalfont{(KL-Divergence Upper Bound, KLUB)}}
Consider an interval $[s, t]$ ($s > t$). If both CTMCs start from the same initial distribution, $\pi_s=\mathbb{P}_s = \mathbb{Q}_s$, then we have:
\begin{equation}
\mathcal{D}_{\mathrm{KL}}(\mathbb{P}_t \| \mathbb{Q}_t) 
\le \mathcal{D}_{\mathrm{KL}}(\mathbb{P}_{\mathrm{paths}} \| \mathbb{Q}_{\mathrm{paths}}) 
= 
\underbrace{\mathbb{E}_{\mathbb{P}_{\mathrm{paths}}} \left[
\sum_{i \ne j} \sum_{\substack{t < u \le s}} H_{u}^{ij} \log \frac{R^1_u(i,j)}{R^2_u(i,j)}
\right]}_{\triangleq~\textrm{\normalfont{KLUB}}_{\pi_s}(\mathbb{P}_t \| \mathbb{Q}_t)}. 
\label{equation:KLUB}
\end{equation}

Here, $\mathbb{P}_t$ and $\mathbb{Q}_t$ are the probability distributions at time $t$ resulting from CTMC 1 and CTMC 2, respectively. $\mathbb{P}_{\mathrm{paths}}$ and $\mathbb{Q}_{\mathrm{paths}}$ denote the distributions over their path spaces $(X_u)_{u \in [t, s]}$, generated by CTMC 1 and CTMC 2, respectively. The indicator function $H_u^{ij}$ is defined as $H_u^{ij} = 1$ if a transition from state $i$ to $j$ occurs at time $u$, and $H_u^{ij} = 0$ otherwise.

\end{theorem}

We denote the rightmost term in \eref{equation:KLUB} as the \emph{Kullback-Leibler Divergence Upper Bound (KLUB)}, which quantifies the mismatch between distributions generated by different CTMCs based on their rate matrices. Theorem~\ref{theorem:2} demonstrates that the difference in trajectories sampled from different CTMCs depends on the \textbf{ratio of their rate matrices}. 

Consider CTMC 1 as the ground truth reverse CTMC and CTMC 2 as the reverse CTMC obtained via parallel sampling by substituting the forward transition kernel with the corresponding reverse-in-time kernel. Thus, $\mathcal{E}_{\mathrm{CDE}}(s \shortrightarrow t | \mathbf{x}_s)$ can be expressed as the KL divergence between the output distributions of the two CTMCs at time $t$, starting from the initial point $\mathbf{x}_s$. This leads to the upper bound (Proof in \aref{appendix:technique0}):
% \begin{equation}
%     \mathcal{E}_{\mathrm{CDE}}(s \shortrightarrow t | \mathbf{x}_s) \le \textrm{\normalfont{KLUB}}_{\mathbf{x}_s}(\mathbb{P}_t \| \mathbb{Q}_t).
% \label{equation:CDEbound}
% \end{equation}
% 
% Notably, samplers that use parallel sampling can also be represented in the form of a CTMC. For example, in \(\tau\)-leaping, the rate matrix is approximately constant and represented as \( \tilde{R}'_u = \tilde{R}_t \) for \( u \in [t, t - \tau) \). Here, \( \tilde{R}' \) denotes the reverse rate matrix used by \(\tau\)-leaping, while \( \tilde{R}_t \) refers to the actual reverse rate matrix at time \( t \). As a result, we can use the theorem above to quantify the compounding error of parallel sampling by computing the KLUB between the outputs of the exact reverse CTMC and the reverse CTMC using parallel sampling.
% 
% While the \eref{equation:CDEbound} allows us to estimate the CDE starting from a specific state \(\mathbf{x}_s\), in practice, we are interested in the average compounding error over all possible starting states at time \( s \). To assess the overall impact of the CDE when transitioning over the interval \( [s, t] \), we consider the expected value of \(\mathcal{E}_{\mathrm{CDE}}(s \shortrightarrow t | \mathbf{x}_s)\) with respect to \(\mathbf{x}_s \sim \mathbb{P}_s\). This leads us to define:
\begin{equation}
\mathcal E_{\mathrm{CDE}}(s \shortrightarrow t) 
% \triangleq \mathbb{E}_{\mathbf{x}_s} \left[ \mathcal E_{\mathrm{CDE}}(s \shortrightarrow t | \mathbf{x}_s) \right]
\le 
% \mathbb{E}_{\mathbf{x}_s} \left[ 
% \textrm{\normalfont{KLUB}}_{\mathbf{x}_s}(\mathbb{P}_t \| \mathbb{Q}_t^{s\shortto t})\right]
% =
\textrm{\normalfont{KLUB}}_{\mathbb{P}_s}(\mathbb{P}_t \| \mathbb{Q}^{s\shortto t}_t),
\label{equation:CDEbound}
\end{equation}
{where $\mathbb{Q}^{s\shortto t}_t$ is given by the process discretized at time $s$ and $t$.
}
% {\color{red}
% Note: we only defined CDE for the true measure $\mathbb{P}$ and the choice of timesteps,
% so the above notation is not valid for general $\mathbb{Q}$.
% }
This expectation considers the distribution of states at time \( s \) and offers a more comprehensive measure of the CDE over the interval \( [s, t] \).
{
Moreover, from the additivity of KLUB, we can bound the sum of CDEs as
\begin{equation}
    \mathcal E_{\mathrm{CDE}}(s \shortrightarrow t) + \mathcal E_{\mathrm{CDE}}(t \shortrightarrow u)
    \le \textrm{\normalfont{KLUB}}_{\mathbb{P}_s}(\mathbb{P}_t \| \mathbb{Q}^{s\shortto t}_t)
    + \textrm{\normalfont{KLUB}}_{\mathbb{P}_t}(\mathbb{P}_u \| \mathbb{Q}^{t\shortto u}_u)
    = \textrm{\normalfont{KLUB}}_{\mathbb{P}_s}(\mathbb{P}_u \| \mathbb{Q}^{s\shortto t\shortto u}_u),
\label{equation:CDE-KLUB}
\end{equation}
which shows that the KLUB can be useful for comparing the quality of
discretization of the interval $[s, u]$ with different break point $t$. 
}
\yh{We can easily extend this result for the sum of CDEs over the entire sampling schedule $\{ T = t_0 \shortto t_1 \shortto \dots \shortto t_{N-1} \shortto 0 = t_{N} \}$.}
\yh{
% Along the path \( s \shortto t \shortto u \), combining Eqs.~\eqref{eq:KL_CDE} and \eqref{equation:CDE-KLUB}, we obtain
% \begin{align*}
%     \mathcal{D}_{\mathrm{KL}}(\mathbb{P}_u \| \mathbb{Q}_u^{s\shortto t \shortto u}) \le \mathcal E_{\mathrm{CDE}}(s \shortrightarrow t) + \mathcal E_{\mathrm{CDE}}(t \shortrightarrow u) \le \textrm{\normalfont{KLUB}}_{\mathbb{P}_s}(\mathbb{P}_u \| \mathbb{Q}^{s\shortto t\shortto u}_u).
% \end{align*}
% Now, let’s extend this CDE to scenarios involving two or more consecutive jumps. For example, consider the case where sampling occurs through \( s \to u \to t \). If we ignore the accumulated error from the first jump that affects the second jump, then the CDE has the following upper bound.
\begin{equation}
    \mathcal{D}_{\mathrm{KL}}(\mathbb{P}_0 \| \mathbb{Q}_0^{T\shortto t_1 \shortto \cdots \shortto 0}) 
    \le 
    \sum_{i=0}^{N-1} \mathcal{E}_{\mathrm{CDE}}(t_i \shortrightarrow t_{i+1})
    \le 
    \textrm{\normalfont{KLUB}}_{\mathbb{P}_T}(\mathbb{P}_0 \| \mathbb{Q}^{T\shortto t_1 \shortto \cdots \shortto 0}_0)
    ,
\end{equation}
where inequality on the left-hand side comes from Theorem~\ref{theorem:1}. 
% \begin{equation}
% \mathcal E_{\mathrm{CDE}}(s \shortrightarrow u \shortrightarrow t) 
% \approx 
% \mathcal E_{\mathrm{CDE}}(s \shortrightarrow u ) + \mathcal E_{\mathrm{CDE}}(u \shortrightarrow t) 
% \le 
% \textrm{\normalfont{KLUB}}_{\mathbb{P}_s}(\mathbb{P}_t \| \mathbb{Q}_t)
% \end{equation}

To summarize, optimizing the sampling schedule involves finding a set of timesteps \( \{ T \shortto t_1 \shortto \dots \shortto t_{N-1} \shortto 0 \} \) that minimizes the KLUB on the right-hand side. This approach approximately reduces the cumulative CDE (middle) and provides an upper bound on the KL divergence between the true distribution and the sampled distribution for the given schedule (left-hand side).

% \jcc{my suggestion is to link all pieces here, as the ultimate goal is to reduce $\mathcal{D}_{\mathrm{KL}}(\mathbb{P}_0 \| \mathbb{Q}_0^{T\shortto t_1 \shortto \cdots \shortto 0})$.}
} 

\subsection{Feasible Computation with Hierarchical Breakdown Strategy}
\label{subsection:section3.3}

Using the derived KLUB, we can formulate the timestep search as a minimization problem over KLUB. Here, we employ a hierarchical breakdown strategy, dividing a coarser sampling schedule into a finer one, as shown in \fref{figure:hierarchical-breakdown}. Suppose our sampling schedule is given by \(\{ T \shortto t \shortto 0 \}\). Let \(\mathbb{Q}_0^{T \shortto t \shortto 0}\) represent the distribution generated by this schedule. Our goal is to find the optimal \(t\) that minimizes cumulative CDE, i.e., 
$\mathcal{E}_{\mathrm{CDE}}(T \shortrightarrow t) + \mathcal{E}_{\mathrm{CDE}}(t \shortrightarrow 0)$. 
This is approximately achievable by minimizing its KLUB upper bound:
\begin{equation}
t_1 = \argmin_{t\in(T,0)} \textrm{\normalfont{KLUB}}(
\mathbb{P}_0 \| \mathbb{Q}_0^{T \shortto t \shortto 0})
\label{equation:optimization}
\end{equation}
With the initial refined interval \(\{ T \shortto t_t \shortto 0 \}\), we seek optimal timesteps \(t_2 \in (T, t_1)\) and \(t_3 \in (t_1, 0)\) by solving the following minimization problems: 
\[
t_2 \in \argmin_{t\in(T,t_1)} \textrm{\normalfont{KLUB}}(
\mathbb{P}_{t_1} \| \mathbb{Q}_{t_1}^{T \shortto t \shortto t_1}) \quad \text{and} \quad t_3 \in \argmin_{t\in(t_1,0)} \textrm{\normalfont{KLUB}}(
\mathbb{P}_{0} \| \mathbb{Q}_{0}^{t_1 \shortto t \shortto 0}).
\] 

\begin{wrapfigure}{r}{7.0cm}
\vspace{-1.25em}
\includegraphics[width=6.5cm]{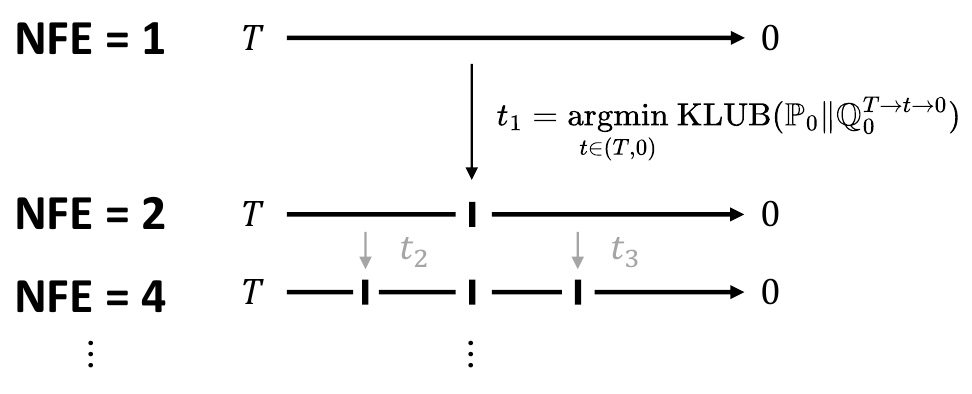}
\caption{
    We optimize the sampling schedule by refining it from coarse intervals to finer intervals, using a hierarchical breakdown strategy.
}
\vspace{-0.5em}
\label{figure:hierarchical-breakdown}
\end{wrapfigure}

The first minimization problem targets matching \(\mathcal{D}_\mathrm{KL}(\mathbb{P}_{t_1} \| \mathbb{Q}_{t_1})\), while the second focuses on matching \(\mathcal{D}_\mathrm{KL}(\mathbb{P}_{0} \| \mathbb{Q}_{0})\).
This results in a further refined sampling schedule \(\{T \shortto t_2 \shortto t_1 \shortto t_3 \shortto 0 \}\). By iterating this process, we continue splitting each interval into smaller ones, optimizing breakpoints using the KLUB criterion. After \(K\) iterations, this hierarchical strategy yields a sampling schedule with \(2^K\) NFEs, optimizing the schedule as the number of steps increases.

% We visualize this recursive process of splitting larger intervals and optimizing each breakpoint to minimize CDE at these endpoints in \fref{figure:hierarchical-breakdown}.

\subsection{Feasible Computation for KLUB Estimation}
\label{subsection:section3.4}

Directly estimating \(\textrm{\normalfont{KLUB}}_{\mathbb{P}_T}(\mathbb{P}_0 \| \mathbb{Q}_0^{T \shortto t \shortto 0})\) by using \eref{equation:KLUB} is practically infeasible due to computational complexities. To address this issue, we propose two techniques that simplify the estimation process.

% There are two main issues: First, sampling a single path \((X_u)_{u \in [T, 0]}\) under \(\mathbb{P}_{\text{paths}}\) is computationally expensive because it requires simulating the exact reverse CTMC from \(T\) to \(0\), which involves a high number of function evaluations (NFEs). Second, in the sampled trajectory \((X_u)_{u \in [T, 0]}\), only the transitions where \(H_{u}^{ij} = 1\) contribute to the expected value calculation in the KLUB. To obtain a reliable estimate, we need to compute the ratio of the rate matrices \( {R^1_u(i,j)}/{R^2_u(i,j)} \) for various \( i, j, u \). However, obtaining sufficient samples that cover all possible transitions across the state space and time is extremely challenging. To address this challenge, we employ two key techniques.

% \begin{figure}[t]
%    \begin{floatrow}
%      \ffigbox[\FBwidth]{% Adjust \FBwidth for width of the figure box
%        \includegraphics[width=0.45\textwidth]{figures-pdf/technique1.pdf}
%      }{
%        \caption{\textbf{Technique1.}}
%        % 
%      }\label{figure:technique1}
%      \ffigbox[\FBwidth]{%
%        \includegraphics[width=0.45\textwidth]{figures-pdf/technique1.pdf}
%      }{%
%        \caption{\textbf{Technique2.} 
%        % 
%        }
%      }\label{figure:technique2}
%    \end{floatrow}
% \end{figure}

\paragraph{Technique 1: Maximizing KL Divergence from a Coarser Approximation}
Instead of minimizing the mismatch between the ground truth distribution \(\mathbb{P}_\text{paths}\) and \(\mathbb{Q}_\text{paths}^{T \shortto t \shortto 0}\), we choose to maximize the discrepancy between \(\mathbb{Q}_\text{paths}^{T \shortto t \shortto 0}\) and a simpler, coarser approximation \(\mathbb{Q}_\text{paths}^{T \shortto 0}\). The key insight is that maximizing this divergence, we can find the optimal sampling time \(t\), which helps reduce the compounding error relative to the true distribution. This relationship can be approximated as (\aref{appendix:technique1}):

\[
\mathcal{D}_\mathrm{KL}(\mathbb{Q}_\text{paths}^{T\shortto t \shortto 0} \| \mathbb{Q}_\text{paths}^{T\shortto 0})
\approx 
\mathcal{D}_\mathrm{KL}(\mathbb{P}_\text{paths} \| \mathbb{Q}_\text{paths}^{T\shortto 0})
- \mathcal{D}_\mathrm{KL}(\mathbb{P}_\text{paths} \| \mathbb{Q}_\text{paths}^{T\shortto t \shortto 0}),
\]

Thus, by maximizing the divergence on the left-hand side, we effectively minimize the discrepancy \(\mathcal{D}_\mathrm{KL}(\mathbb{P}_\text{paths} \,\big\|\, \mathbb{Q}_\text{paths}^{T\shortto t \shortto 0})\) between the true distribution \(\mathbb{P}_\text{paths}\) and the optimized schedule \(\mathbb{Q}_\text{paths}^{T\shortto t \shortto 0}\).

Based on this, we can rewrite our optimization objective as follows (Proof in \aref{appendix:technique1}):
\begin{equation}
t^*
= \argmax_{t\in(T,0)} 
\textrm{\normalfont{KLUB}}_{\mathbb{Q}_T}(
\mathbb{Q}_0^{T \shortto t \shortto 0} \| \mathbb{Q}_0^{T \shortto 0})
= \argmax_{t\in(T,0)} \mathbb{E}_{\mathbb{Q}_{\text{paths}}^{T\shortto t \shortto 0}} \left[
\sum_{i \ne j} \log \frac{R_t(i,j)}{R_T(i,j)} 
\sum_{0 < u \le t} H_{u}^{ij} 
\right]
\label{equation:technique1}
\end{equation}

Equation \eref{equation:technique1} offers a significant computational advantage over \eref{equation:KLUB} by eliminating the need to compute the rate matrix at every transition time. Note that, calculating the reverse rate matrix involves neural network evaluations. Given a trajectory \((X_u)_{u \in [T, 0]}\), \eref{equation:KLUB} requires computing the rate matrix \(R_u\) for all transitions where \(H_u^{ij} = 1\). In contrast, \eref{equation:technique1} decouples the rate matrix from the transition times \(u\), requiring computations only at times \(T\) and \(t\). Moreover, between times \(T\) and \(t\), the rate matrices for both \(\mathbb{Q}^{T \shortto t \shortto 0}_{\text{paths}}\) and \(\mathbb{Q}^{T \shortto 0}_{\text{paths}}\) are equal to \(R_T\), eliminating the need for additional computations in this interval.

\paragraph{Technique 2: Applying the Law of Total Expectation} 
Instead of directly compute the expectation over $\mathbb{Q}_{\text{paths}}^{T\shortto t \shortto 0}$, we compute expectation of conditional expectation, using the law of total expectation, i.e., $\mathbb{E}[X] = \mathbb{E}[\mathbb{E}[X|Y]]$: 
\[
\mathbb{E}_{\mathbb{Q}_{\text{paths}}^{T\shortto t \shortto 0}}
\left[ 
\sum_{i \ne j} \log \frac{R_t(i,j)}{R_T(i,j)} 
\sum_{0 < u \le t} H_{u}^{ij} 
\right] = 
\mathbb{E}_{\mathbb{Q}_{\text{paths}}^{T\shortto t}}
\left[ 
\mathbb{E}_{\mathbb{Q}_{\text{paths}}^{t \shortto 0}}
\left[ 
\sum_{i \ne j} \log \frac{R_t(i,j)}{R_T(i,j)} 
\sum_{0 < u \le t} H_{u}^{ij} 
\Bigg| X_t = i
\right]
\right]
\]

On the left side, to compute the expected value, we would need to sample the entire trajectory \((X_u)_{u \in [T, 0]}\) and then calculate the inner sum. On the right side, this process is broken down into two steps: first, we sample a trajectory from \(X_T\) to \(X_t\), and then, given \(X_t\), we sample from \(X_t\) to \(X_0\) to compute the inner expectation. Our main insight is that the expected value in the second step can be obtained in closed form.

The term \(\sum_{0 < u \le t} H_u^{ij}\) counts the number of transitions from state \(i\) to state \(j\) in the interval \([t, 0)\). Since \(\mathbb{Q}_{\text{paths}}^{t \shortto 0}\) uses a single rate matrix over this interval, for \(i \ne j\), the expected number of transitions from \(i\) to \(j\) is approximately given by
\citep[Sec.~4.3]{campbell2022continuous}:
\[
\mathbb{E}_{\mathbb{Q}_{\text{paths}}^{t \shortto 0}}
\left[\sum_{0 < u \le t} H_{u}^{ij}\right] 
\approx
R_t(i, j) \Delta t
\]
where $\Delta t$ represents the size of the time interval.

Using this result, we can simplify Equation \eref{equation:technique1} as follows (Derivation in \aref{appendix:technique2}):
\begin{equation}
\textrm{KLUB}_{\mathbb{Q}_T} \big(\mathbb{Q}_0^{T \shortto t \shortto 0} \,\big\|\, \mathbb{Q}_0^{T \shortto 0}\big) 
\approx \mathbb{E}_{\mathbb{Q}_{\text{paths}}^{T \shortto t}}
\left[
\sum_{j \ne X_t} \log \frac{R_t(X_t, j)}{R_T(X_t, j)} \times R_t(X_t, j) \Delta t  
\right].
\label{equation:technique2}
\end{equation}

\begin{wrapfigure}{r}{4.75cm}
\vspace{-1.5em}
\includegraphics[width=4.75cm]{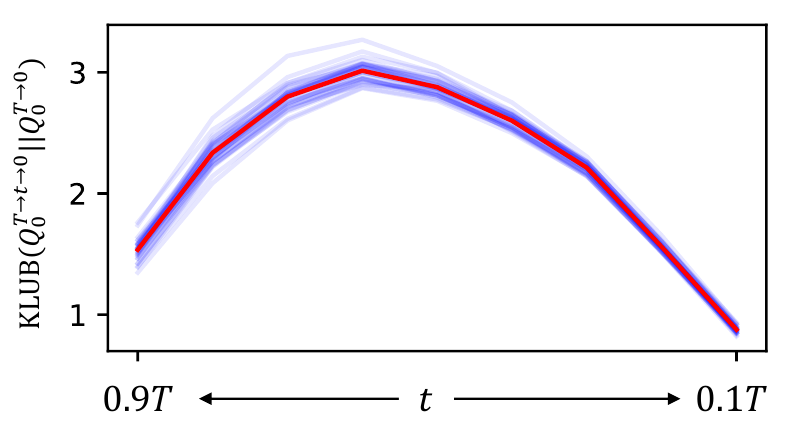}
\caption{
    The values of 
    KLUB with respect to \(t\).
    Blue lines show estimated results from individual $(X_t)_{t\in[T,0]}$, while the red line is the average.
}
\label{figure:klub-distribution}
\end{wrapfigure}

In \eref{equation:technique1}, obtaining a reliable KLUB estimation requires sampling trajectories to capture transitions between various states \(i\) and \(j\), which is both inaccurate and sample-inefficient—especially when dealing with large state spaces. For instance, in text generation tasks where the state space can be around 50,257, it's practically impossible to estimate the transition ratios between all pairs \((i, j)\) through sampling alone. In contrast, \eref{equation:technique2} allows us to compute this component in closed form, leading to more reliable and efficient calculations. The full algorithm for KLUB computation, combining Techniques 1 and 2, is provided in \aref{appendix:algorithm1}.

Now, we can maximize \eref{equation:technique2} using a standard optimization algorithm. Before selecting the algorithm, we conduct a preliminary check to observe how the KLUB value changes with respect to \(t\) (\fref{figure:klub-distribution}). Interestingly, it exhibits a unimodal shape. Since we are working with a single variable \(t\) in the unimodal optimization landscape, we use the golden section search \citep{press2007numerical}, a well-known one-dimensional search algorithm, to find the value of \(t\) that maximizes KLUB. This method has the advantage of not relying on hyperparameters like learning rate, which can significantly affect performance.

\section{Experiments}

In this section, we evaluate the \textit{Jump Your Steps} (JYS) sampling schedule across various datasets and models. We compare the JYS schedule with the uniform sampling schedule, which sets all intervals to the same size. Except for the Countdown dataset, we use open-sourced pretrained models for our experiments. It is important to note that the Gillespie algorithm is only applicable to an absorbing transition matrix, as uniform or Gaussian transition kernels do not have a fixed number of transitions. For further experimental details and additional qualitative results, please refer to the Appendix~\ref{appendix:experiment-details} and \ref{appendix:additional-results}.

\begin{wrapfigure}{r}{5.5cm}
\vspace{-3.0em}
\includegraphics[width=5.5cm]{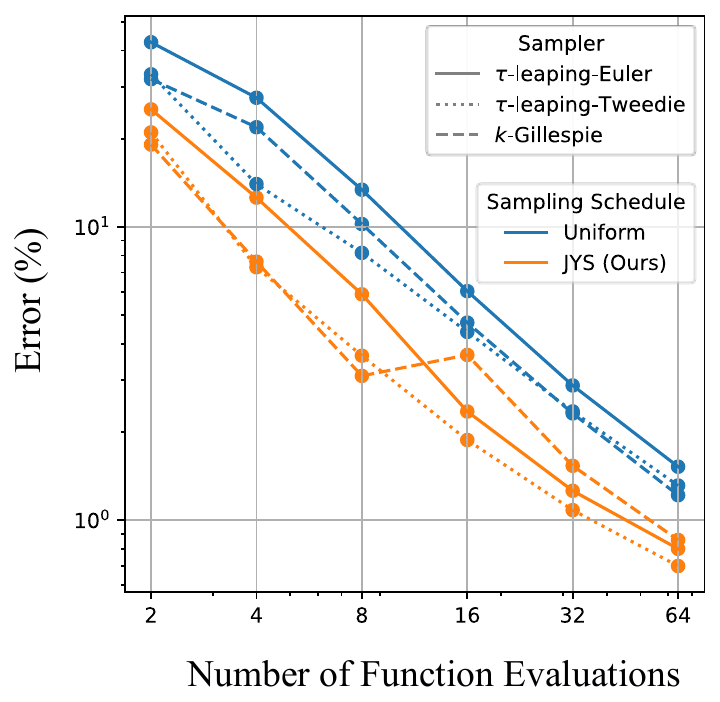}
\caption{
    Performance comparisons on Countdown. 
    The JYS schedule enhances sampling quality across different types of samplers.
}
\vspace{-0.5em}
\label{figure:countdown}
\end{wrapfigure}

\subsection{The CountDown Dataset}

Following \cite{zhao2024informed}, to evaluate our sampling schedule performance, we created a synthetic sequence dataset with a strong position-wise correlation structure. Each sample consists of 256 tokens, and each token has a value between 0 and 31. Each data sequence \( X^{0:255} \) is generated according to the following rules:

\begin{align*}
X^0 
&\sim \text{Uniform} \{1, \ldots, S\}, \\
X^{d+1} \mid X^d &\sim 
\begin{cases}
\delta_{X^d-1} & \text{if } X^d \neq 0 \\
\text{Uniform} \{1, \ldots, S\} & \text{if } X^d = 0
\end{cases}
\end{align*}

We trained a SEDD \citep{lou2023discrete} with an absorb transition matrix on this generated data. We measure the model performance by the proportion of generated samples that violated the rule, i.e., failed to count downwards from the previous token. The results are shown in \fref{figure:countdown}. We observe that the JYS schedule has fewer errors compared to the uniform schedule for the same NFE.

% As shown in \fref{figure:jump-your-steps}, the JYS schedule allocates more sampling steps to the initial regions where mutual information between tokens is high, thereby reducing compounding errors. 

\subsection{CIFAR-10}

\begin{figure}[t]
    \centering
    \includegraphics[width=1\linewidth]{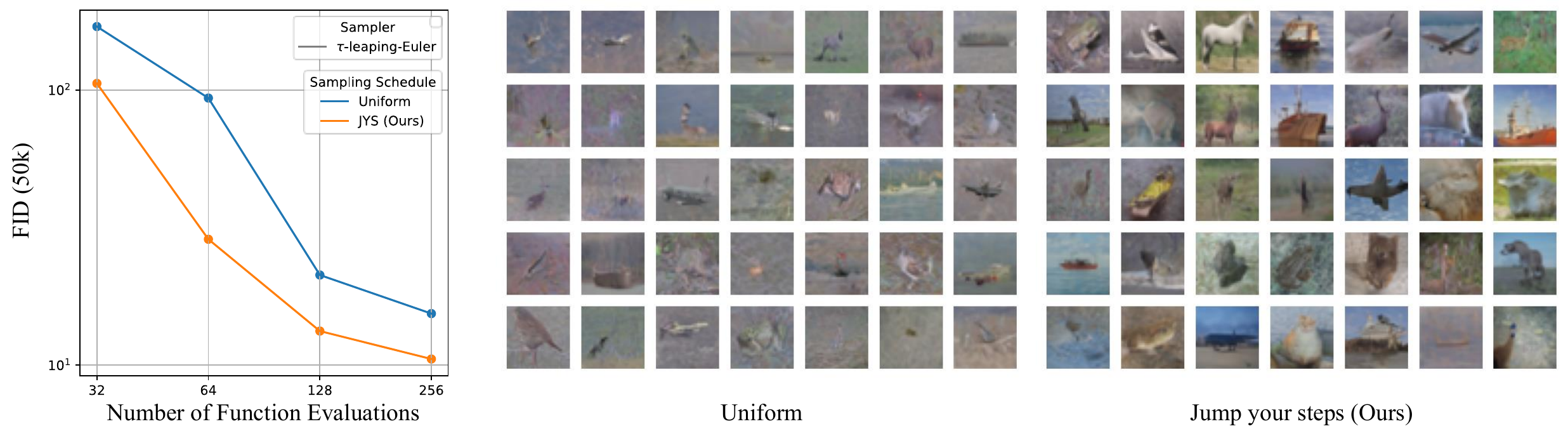}
    \caption{
    \figleft~Performance comparisons on CIFAR-10.
    \figright~Samples generated using the uniform and JYS schedules, both with NFE 64.
    }
    \label{figure:CIFAR-10-qual}
\end{figure}

We demonstrate our sampling schedule in the image domain. For this experiment, we use a pretrained model from \cite{campbell2022continuous}, which employs a gaussian transition matrix and denoising parameterization. Each data sample is a flattened image with a length of $3\times32\times32$, composed of tokens with values ranging from 0 to 255.

\fref{figure:CIFAR-10-qual} (left) shows the FID score using 50k samples with the number of function evaluations (NFE) from 32 to 256. We observe that, for all NFEs, the JYS schedule yields a better FID score at the same NFE compared to the uniform schedule. \fref{figure:CIFAR-10-qual} (right) shows randomly generated unconditional CIFAR-10 samples with NFE = 64. The uniform schedule produces blurry images, whereas the images generated using the JYS schedule exhibit clearer colors and shapes of objects. 
% More samples can be found in the appendix. 

\subsection{Monophonic Music}

We test our method on conditional music generation using the Lakh pianoroll dataset \citep{raffel2016learning, dong2018musegan}. For this experiment, we employ a pretrained model from \cite{campbell2022continuous}, which uses a uniform transition matrix and denoising parameterization. Each data sequence contains 256 timesteps (16 per bar), and we measure performance by conditioning on two bars to generate the remaining 14 bars, following the setup in \cite{campbell2022continuous}.

% \begin{wrapfigure}{r}{5.3cm}
% \vspace{-2em}
% \includegraphics[width=5.3cm]{figures-pdf/piano.pdf}
% \caption{
%     Performance comparisons on Monophonic music. 
% }
% \label{figure:piano}
% \end{wrapfigure}

We evaluate how different the generated results were when using a smaller NFE (from 2 to 64), compare to samples generated with an NFE of 512. Specifically, we calculate the Hellinger distance between the note distributions in the generated samples. The results are presented in \fref{figure:piano}. 
Given the same NFE, we observe that samples generated using our method were more similar to those generated with a high NFE.

% \begin{wrapfigure}{r}{9cm}
% \vspace{-2em}
% \includegraphics[width=9cm]{figures-pdf/text-quant.pdf}
% \caption{
%     Performance comparisons on text generation. 
%     % Generative perplexity is measured by using GPT-2-large.
% }
% \vspace{-0.5em}
% \label{figure:text-uncond}
% \end{wrapfigure}

\begin{figure}[t]
   \begin{floatrow}
     \ffigbox[\FBwidth]{% Adjust \FBwidth for width of the figure box
       \includegraphics[width=0.35\textwidth]{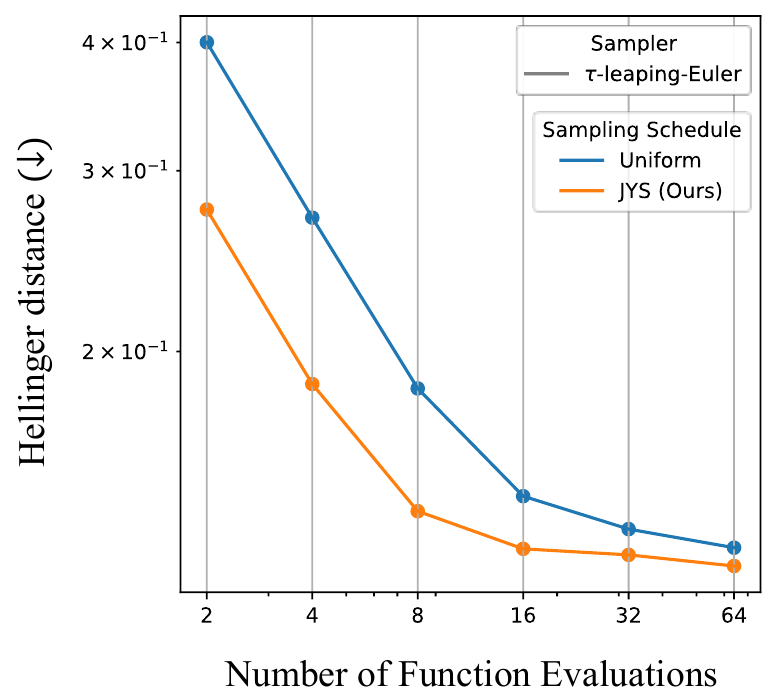} 
     }{
       \caption{Performance comparisons on Monophonic music.}
       \label{figure:piano}
     }
     \ffigbox[\FBwidth]{%
       \includegraphics[width=0.61\textwidth]{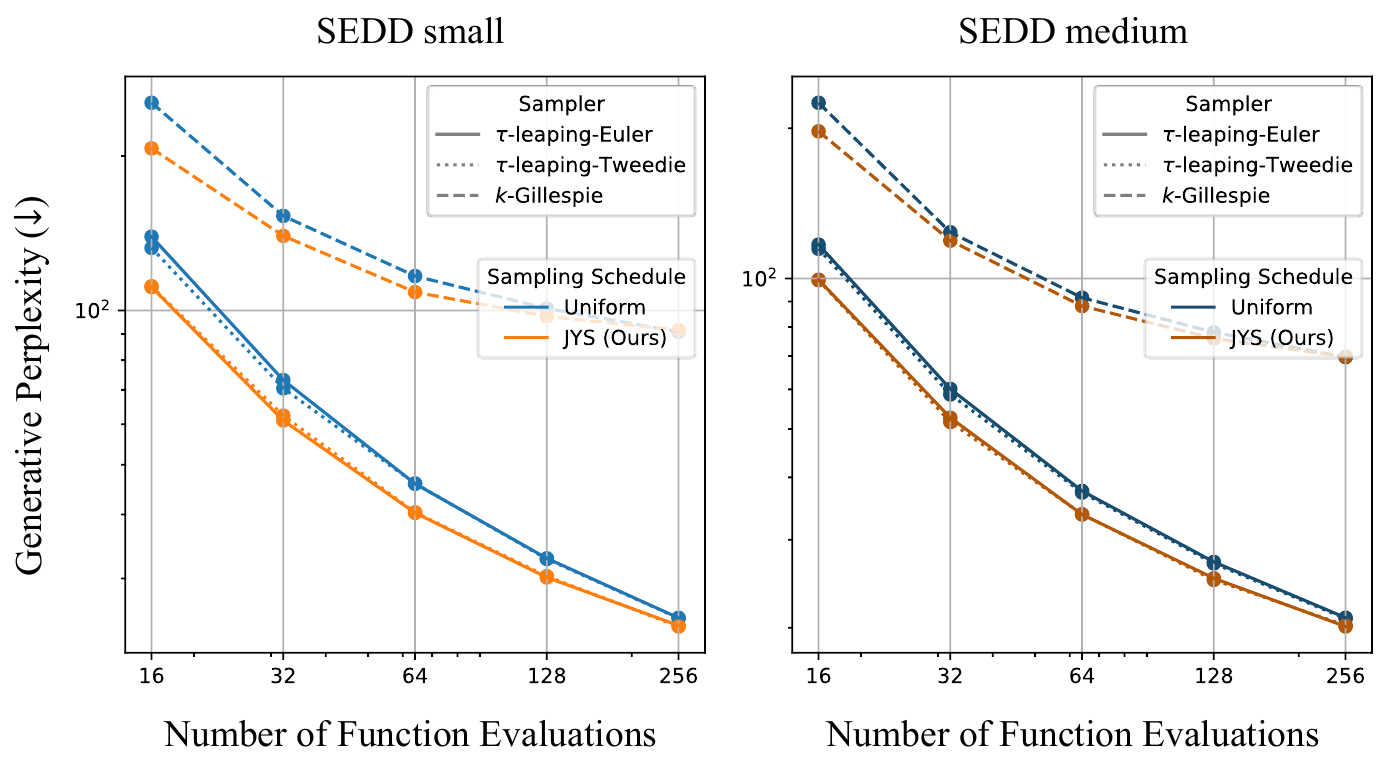}
     }{
        \caption{
        Performance comparisons on text generation. 
        Generative perplexity is measured by using GPT-2-large.
        }
        \label{figure:text-uncond}
     }
   \end{floatrow}
\end{figure}

\subsection{Text modeling}

Finally, we validate our method on text generation. For this experiment, we use a pretrained model from \cite{lou2023discrete}, which employs an absorbing transition matrix and score-based parameterization. We use two model sizes, SEDD-small and SEDD-medium, in the experiments; both models use the GPT-2 tokenizer and were trained on OpenWebText. The JYS schedule, optimized on SEDD-small, is also used for the experiments with SEDD-medium.

Following \cite{lou2023discrete}, we measure the generative perplexity of sampled sequences (using a GPT-2 large for evaluation). We generated 1,024 samples and each sample constructed with sequences of 1,024 tokens. We simulate 16 to 256 NFE for generation. \fref{figure:text-uncond} shows the results, demonstrating better perplexity at the same NFE.

\begin{figure}[t]
    \centering
    \includegraphics[width=1\linewidth]{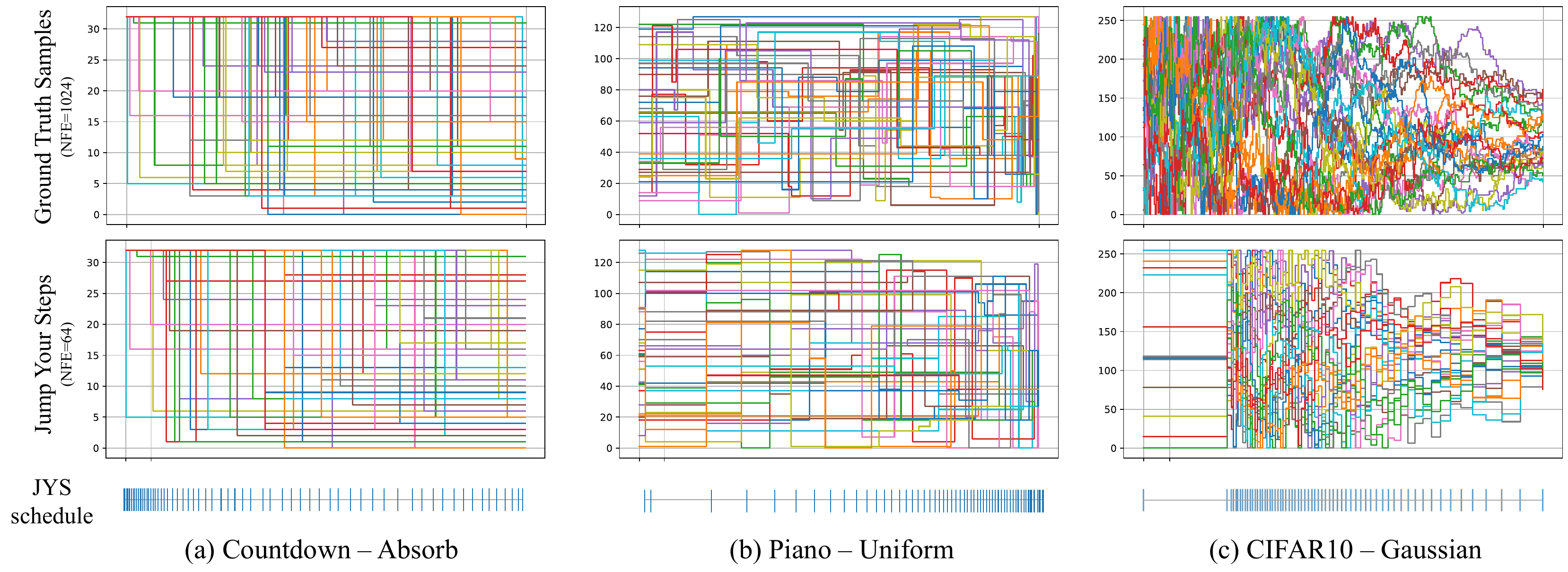}
    \caption{
    Sampling trajectories for different dataset-transition matrix combinations.
    \figtop{} Trajectories using infinitesimal timesteps.
    \figmiddle{} Trajectories using the JYS schedule.
    \figbottom{} Optimized sampling schedule with JYS.
    % Sampling trajectories for various dataset-transition matrices. \figtop{} Sampling trajectories using infinitesimal step \figmiddle{} Sampling trajectories using JYS schedule. \figbottom{} Optimized sampling schedule using JYS.
    }
    \label{figure:jump-your-steps}
\end{figure}

\subsection{Characteristics of Jump-Your-Step Sampling schedule}

In \sref{subsection:section3.1}, we hypothesized that in regions where the conditional mutual information is low, the CDE would also be small, allowing steps to be skipped with minimal performance degradation. Here, we aim to verify if the JYS operates as expected according to this hypothesis.

\fref{figure:jump-your-steps} shows the JYS sampling schedules optimized for various transition matrices. First, in the \textit{absorb} case (Left), as discussed in \fref{fig:teaser} (Bottom), we observe that large intervals are concentrated toward the latter part of the process. This occurs because the previously generated tokens help reduce the uncertainty of other tokens. In contrast, in the \textit{uniform} case (Middle), large intervals appear at the beginning. This can be understood as a result of \( X_t \) following a uniform distribution, making the tokens independent, leading to lower mutual information for larger \( t \). Lastly, for the \textit{Gaussian} transition (Right), large intervals appear initially and then increase again over time. This pattern suggests that, initially, tokens behave independently like in the uniform case, but after a certain timestep, the effect of resolving uncertainty, similar to the absorb case, becomes more significant as more tokens are generated.

% 우리는 conditional mutual information 이 작은 구간이 CDE 가 작을 것이고, 따라서 여기서 step 을 건너뛰어도 상대적으로 performance 하락이 덜할 것이라는 가설을 세웠는데, 실제 JYS 가 우리의 가설대로 작동하는지 확인하려 한다. 

% \fref{figure:jump-your-steps} 는 각 dataset 에서 학습된 JYS sampling schedule 을 보여주고 있다. \figleft{} 먼저 absorb 의 경우 \fref{fig:teaser} Bottom 에서 다뤘던것처럼 상대적으로 뒷 부분에 large interval 이 분포된 것을 확인할 수 있다. 이는 이미 생성된 token 들이 다른 token 의 uncertainty 를 해소했기 때문. 반면 \figmiddle{} Uniform 의 경우 초반부에 large interval 이 분포함. 이는 X_t 가 uniform distribution 이라서 independent 하기 때문에 mutual information 이 large t 에서 작아서 발생하는 현상으로 이해할 수 있다. 마지막으로 \figright{} Gaussian transition 의 경우에는 초반에 interval 이 존재하고, 다시 점점 interval 이 커지는 양상. 이는 초반에는 uniform 처럼 각 token 이 independent 했다가, 특정 timestep 을 지나서는 absorb 처럼 생성된 토큰이 uncertainty 를 해소하는 효과가 더 커져서 나타나는 양상으로 이해할 수 있다. 

These observations demonstrate how the JYS schedule adapts to the underlying data distribution and token dependencies, effectively allocating computational resources where they are most needed to minimize the compounding error during parallel sampling.

\section{Related Work}

\subsection{Efficient Sampling for Continuous Diffusion Models}

After the seminal work by \cite{song2020score}, which interpreted diffusion models as a Stochastic or Ordinary Differential Equations (SDE/ODE), various SDE~\citep{jolicoeur2021gotta, xu2023restart} and ODE solvers~\citep{song2020denoising, lu2022dpm, zhang2022fast, dockhorn2022genie, liu2022pseudo, zheng2023dpm} have been proposed to improve sampling speed. 

The work most closely related to ours is ``Align Your Step''~\citep{sabour2024align}, which focuses on sampling schedule optimization in continuous diffusion models. In contrast, our approach targets discrete diffusion models, where we derive the KLUB for CTMCs to optimize the sampling schedule. We also propose a computationally efficient algorithm for KLUB computation and optimization.

\subsection{Discrete Diffusion Models}

Several approaches have been developed for training discrete diffusion models, including denoising parameterization~\citep{austin2021structured, campbell2022continuous, gu2022vector, gat2024discrete, campbell2024generative, shi2024simplified} and score parameterization~\citep{sun2022score, meng2022concrete, lou2023discrete}. Recently, SEDD has outperformed GPT-2 in text modeling, gaining traction as an alternative to autoregressive models~\citep{deschenaux2024promises}.

In terms of sampling, two main directions have emerged. The first focuses on efficient sampling, with $\tau$-leaping for CTMC and methods like analytic sampling~\citep{sun2022score}, Tweedie sampling \citep{lou2023discrete}, and $k$-Gillespie \citep{zhao2024informed} improving accuracy. The second aims to reduce compounding error via corrector steps, such as random correctors \citep{campbell2022continuous}, separate corrector training \citep{lezama2022discrete}, or enabling the model to act as an informed corrector \citep{zhao2024informed}.

These methods complement the sampling schedule optimization explored in this paper and can be used together for further improvements.

\section{Conclusions}

We present \textit{Jump Your Steps}, a principled method designed to optimize the sampling schedule and minimize these numerical errors without incurring additional computational costs during inference. Unlike existing approaches that rely on extra computational efforts, such as predictor-corrector methods, our technique operates independently and efficiently. Through extensive evaluations on synthetic and real-world datasets—including monophonic piano, image, and text generation—we demonstrate that our method consistently enhances performance across different transition kernels in discrete diffusion models and effectively complements various samplers.

\newpage
\bibliography{iclr2025_conference}
\bibliographystyle{iclr2025_conference}

\newpage

\appendix

\begingroup
\hypersetup{colorlinks=false, linkcolor=red}
\hypersetup{pdfborder={0 0 0}}
\renewcommand\ptctitle{}

\part{Appendix} % Start the appendix part
\parttoc % Insert the appendix TOC

% \renewcommand{\contentsname}{Table of Contents for Appendix}
% \tableofcontents

\endgroup

\newcommand{\pp}{\mathbb{P}}
\newcommand{\qq}{\mathbb{Q}}

\section{Theoretical Details}

\subsection{Proof of Theorem 3.1}
\label{appendix:theorem1}
Although we treated the two-dimensional case where we have $X_t=(X^1_t, X^2_t)$ in the main body,
we consider the general $d$-dimensional case with $X_t=(X_t^1, \ldots, X_t^d)$.
In that case, the definition of $\mathcal{E}_\mathrm{CDE}$ is given by
\begin{equation}
    \mathcal{E}_\mathrm{CDE}(s\to t|\mathbf{x}_s)\triangleq
    \mathcal{D}_\mathrm{KL}(P_{\mathbf{x}_t|\mathbf{x}_s} \Vert P_{X_t^1|\mathbf{x}_s} \otimes\cdots\otimes P_{X_t^d|\mathbf{x}_s}).
    \label{eq:d-dim-cde}
\end{equation}

In general, for a discrete probability distribution $\pp_\mathrm{prior}(\cdot)$
over the space $\mathcal{S}$
and a conditional distribution (or denoiser) $\pp_\mathrm{cond}(\cdot|\cdot)$ over the same space,
i,e, $\pp_\mathrm{cond}:\mathcal{S}\times\mathcal{S}\to\mathbb{R}$,
let us just write the resulting distribution as follows:
\begin{equation*}
    \pp_\mathrm{cond}\pp_\mathrm{prior}\triangleq
    \mathbb{E}_{\mathbf{x}\sim\pp_\mathrm{prior}}[\pp_\mathrm{cond}(\cdot|\mathbf{x})]
    =
    \sum_{x\in\mathcal{S}}\pp_\mathrm{cond}(\cdot|x)\pp_\mathrm{prior}(x).
\end{equation*}
Then, if we define $\pp_{t_{i+1}|t_i}(\cdot|\mathbf{x}_{t_i})\triangleq P_{X_{t_{i+1}}|\mathbf{x}_{t_i}}$
and $\qq_{t_{i+1}|t_i}(\cdot|\mathbf{x}_{t_i})\triangleq P_{X_t^1|\mathbf{x}_s} \otimes\cdots\otimes P_{X_t^d|\mathbf{x}_s}$
following \eqref{eq:d-dim-cde},
we can denote the target distributions in the theorem as follows:
\begin{equation*}
    \pp_0 = \pp_{t_N|t_{N-1}}\cdots \pp_{t_1|t_0}\pp_{t_0},
    \qquad
    \qq_0
    \triangleq \mathbb{Q}_0^{T\shortto t_1 \shortto \cdots \shortto 0}
    =
    \qq_{t_N|t_{N-1}}\cdots \qq_{t_1|t_0}\pp_{t_0}.
\end{equation*}
For simplicity,
let us also define the mid-time distributions as
\begin{equation*}
    \pp_{t_i} \triangleq \pp_{t_i|t_{i-1}}\cdots \pp_{t_1|t_0}\pp_{t_0},
    \qquad
    \qq_{t_i} \triangleq \qq_{t_i|t_{i-1}}\cdots \qq_{t_1|t_0}\pp_{t_0},
    \qquad
    \qq_{t_0} \triangleq \pp_{t_0}.
\end{equation*}

Let us now consider the case where $s=t_{i-1}$ and $t=t_{i}$ for some $i>0$
specifically.
Let
\[
    \pp_{t,s}(y,x)\triangleq\pp_{t|s}(y|x)\pp_s(x),
    \qquad \qq_{t,s}(y,x)\triangleq\qq_{t|s}(y|x)\qq_s(x)
\]
to denote joint distributions on $\mathcal{S}\times\mathcal{S}$.
Then, since $\pp_t=\pp_{t|s}\pp_s$ and $\qq_t=\qq_{t|s}\qq_s$
are marginal distributions of $\pp_{t,s}$ and $\qq_{t,s}$, we have
\begin{align}
    \mathcal{D}_\mathrm{KL}(\pp_t\Vert \qq_t)
    &\le \mathcal{D}_\mathrm{KL}(\pp_{t,s}\Vert\qq_{t,s}) \nonumber\\
    &= \sum_{y, x}\pp_{t|s}(y|x)\pp_s(x)\log\frac{\pp_{t|s}(y|x)\pp_s(x)}{\qq_{t|s}(y|x)\qq_s(x)}\nonumber\\
    &= \sum_{y, x}\pp_{t|s}(y|x)\pp_s(x)\log\frac{\pp_s(x)}{\qq_s(x)}
    + \sum_{y, x}\pp_{t|s}(y|x)\pp_s(x)\log\frac{\pp_{t|s}(y|x)}{\qq_{t|s}(y|x)}\nonumber\\
    &= \mathcal{D}_\mathrm{KL}(\pp_s\Vert\qq_s)
    + \sum_{x}\pp_s(x)\mathcal{D}_\mathrm{KL}(\pp_{t|s}(\cdot|x)\Vert \qq_{t|s}(\cdot|x))\nonumber\\
    &=\mathcal{D}_\mathrm{KL}(\pp_s\Vert\qq_s)
    + \mathbb{E}_{\mathbf{x}_s\sim \pp_s}[\mathcal{E}_\mathrm{CDE}(s\to t|\mathbf{x}_s)]\nonumber\\
    &=\mathcal{D}_\mathrm{KL}(\pp_s\Vert\qq_s) + \mathcal{E}_\mathrm{CDE}(s\to t).
    \label{eq:key-ineq}
\end{align}
By iteratively using \eqref{eq:key-ineq},
we obtain
\begin{align*}
    \mathcal{D}_\mathrm{KL}(\pp_0\Vert \qq_0)
    &= \mathcal{D}_\mathrm{KL}(\pp_{t_N}\Vert \qq_{t_N})\\
    &\le \mathcal{D}_\mathrm{KL}(\pp_{t_{N-1}}\Vert \qq_{t_{N-1}}) + \mathcal{E}_\mathrm{CDE}(t_{N-1}\to t_N)\\
    &\le \mathcal{D}_\mathrm{KL}(\pp_{t_{N-2}}\Vert \qq_{t_{N-2}}) + \mathcal{E}_\mathrm{CDE}(t_{N-2}\to t_{N-1})
        + \mathcal{E}_\mathrm{CDE}(t_{N-1}\to t_N)\\
    &\ \ \vdots\\
    &\le \mathcal{D}_\mathrm{KL}(\pp_{t_0}\Vert \qq_{t_0}) + \sum_{i=0}^{N-1}\mathcal{E}_\mathrm{CDE}(t_{i}\to t_{i+1}).
\end{align*}
Since we set $\pp_{t_0} = \qq_{t_0}$,
we have completed the proof.
%\hfill\qed

\subsection{Proof of Theorem 3.2}
\label{appendix:theorem2}

% To derive KLUB for CTMCs, first we borrowed the change of measure $\frac{d\mathbb{P}_{\mathrm{paths}}}{d\mathbb{Q}_{\mathrm{paths}}}$ for CTMCs from Section 3 of \cite{ding2021markov}. Second, we compute and organize the equation of $D_\mathrm{KL}({\mathbb{Q}_{\mathrm{paths}}||\mathbb{P}_{\mathrm{paths}}}) = \mathbb{E}_{\mathbb{Q}_{\mathrm{paths}}}\left[\frac{d\mathbb{Q}_{\mathrm{paths}}}{d\mathbb{P}_{\mathrm{paths}}}\right]$.

Although we state the theorem for a backward CTMC from time $s$ to time $t$ ($t<s$)
and bound the KL-divergence at time $t$,
here we will just consider the forward CTMC from time $0$ to time $T$ and bound the KL-divergence at time $T$ for simplicity.
After this change of the time direction and interval,
the formal statement and proof of Theorem~\ref{theorem:2} is given in Theroem~\ref{thm:klub-forward-formal}.

To derive the KL-divergence upper bound (KLUB) for continuous-time Markov chains (CTMCs), we first adopt the change of measure $\frac{d\mathbb{P}_{\mathrm{paths}}}{d\mathbb{Q}_{\mathrm{paths}}}$ for CTMCs from Section 3 of \cite{ding2021markov}. Next, we compute and organize the equation for the $D_\mathrm{KL}({\mathbb{P}_{\mathrm{paths}} | \mathbb{Q}_{\mathrm{paths}}}) = \mathbb{E}_{\mathbb{P}_{\mathrm{paths}}}\left[\log\frac{d\mathbb{P}_{\mathrm{paths}}}{d\mathbb{Q}_{\mathrm{paths}}}\right]$.

We consider the following two forward CTMCs over $[0, T]$:
\[
    \begin{cases}
    \textrm{CTMC 1}: & q^1_{u+du|u}(y \mid x) = \delta_{xy} + R^1_t(x,y) du + o(du), \\ 
    \textrm{CTMC 2}: & q^2_{u+du|u}(y \mid x) = \delta_{xy} + R^2_t(x,y) du + o(du).
    \end{cases}
\]
Here, $R^1_t$ and $R^2_t$ represent the rate matrices of each CTMC, with a finite state space $S = \{x_1, \cdots, x_N\}$,
and $du>0$.

We introduce some notations. Define the functions $H_t^i, H_t^{ij}$ as follows:
\[
    H_t^{i}:=\delta(X_t-x_i), 
    \quad H_t^{i j}:= H_t^{j}H_{t-}^{i},
\]
where \(\delta(\cdot)\) denotes the Dirac delta function, and \(t-\) is the left limit of \(t\). By definition, \(H_t^{i j} = 1\) indicates a transition from \(x_i\) to \(x_j\) at time \(t\). The CTMC \((X_t)_{t \in [0, T]}\) is defined as a function from the sample space \(\Omega\) to the path space \(\mathcal{C} \triangleq [0, T] \times S\), i.e., \(X: \Omega \rightarrow \mathcal{C}\).

We define two probability measures over the path space:

\begin{itemize}
  \item $\mathbb{\mathbb{P}}_{\mathrm{paths}}$, under which $(X^1_t)_{t \in [0, T]}$ has the law of CTMC 1. 
  \item $\mathbb{\mathbb{Q}}_{\mathrm{paths}}$, under which $(X^2_t)_{t \in [0, T]}$ has the law of CTMC 2. 
\end{itemize}

First, we adopt the change of measure for CTMCs from \cite{ding2021markov}.

\begin{proposition}
(\citet{ding2021markov}, Eq. 3.2)
Consider a family of bounded real-valued processes $\{\kappa_t(i, j)\}_{i, j\in\{1,\cdots, N\}}$, such that $\kappa_t(i, j) > -1$ and $\kappa_{ii}(t) = 0$. Define $(\eta_t)_{0\le t \le T}$ as 
\[
\eta_t = e^{-L_t} \prod_{0 < u \le T} \left( 1 + \sum_{i,j=1}^N \kappa_u(i, j)H_u^{ij} \right),
\]
where $L_t = \int^t_0 \sum_{i,j=1}^N \kappa_u(i, j)R^1_u(i,j)H^i_u du$. This result implies the existence of a probability measure $\mathbb{Q}$ defined by
\begin{equation}
\label{appendix-eq-Q}
\frac{d\mathbb{Q}}{d\mathbb{P}} = \eta_t.
\end{equation}
\end{proposition}

The following result allows us to define \(\kappa_t(i, j)\) for two CTMCs.

\begin{proposition}
(\citet{ding2021markov}, Eq. 3.4)
For the probability measure $\mathbb{Q}$ defined in Eq.(\ref{appendix-eq-Q}), if $(X^1_t)_{t\in [0,T]}$ is a CTMC under $\mathbb{P}$ with rate matrix $R^1$ and $(X^2_t)_{t\in [0,T]}$ is a CTMC under $\mathbb{Q}$ with rate matrix $R^2$, then $R^2$ satisfies
\begin{align*}
R^2_{ii}(t) &= -\sum_{j\ne i} R^2_t(i, j), \\
R^2_t(i, j) &= (1 + \kappa_t(i, j)) R^1_t(i, j).
\end{align*}
\end{proposition}

From Proposition A.2, defining \(\kappa_t(i, j) = R^2_t(i, j) / R^1_t(i, j) - 1\) (as in Proposition A.1), we can compute the change of measure between the two CTMC-defined measures \(\mathbb{P}_{\mathrm{paths}}\) and \(\mathbb{Q}_{\mathrm{paths}}\) as \(\frac{d\mathbb{P}_{\mathrm{paths}}}{d\mathbb{Q}_{\mathrm{paths}}}\).

Now, we are ready to derive the KLUB for CTMCs:

\begin{theorem}\label{thm:klub-forward-formal}
{\normalfont{(KL-divergence Upper bound, KLUB)}} Consider the following two forward CTMCs: 
\[
    \begin{cases}
    \textrm{CTMC 1}: & q^1_{u+du|u}(y \mid x) = \delta_{xy} + R^1_t(x,y) du + o(du), \\ 
    \textrm{CTMC 2}: & q^2_{u+du|u}(y \mid x) = \delta_{xy} + R^2_t(x,y) du + o(du).
    \end{cases}
\]
Here, $R^1_t$ and $R^2_t$ represent the rate matrices of each CTMC over $[0, T]$, with a finite state space $S = \{x_1, \cdots, x_N\}$.
Let $\mathbb{P}_T$ and $\mathbb{Q}_T$ be the resulting probability distributions at the time $T$ of the outputs of CTMC 1 and 2, respectively. Then we have:
\[
D_\mathrm{KL}(\mathbb{P}_T \| \mathbb{Q}_T) \le \mathbb{E}_{\mathbb{P}_{\mathrm{paths}}} \left[
\sum_{\substack{H_{u}^{ij}=1 \\ 0 < u \le T}}  \log \frac{R^1_u(i,j)}{R^2_u(i,j)}.
\right]
\]
where, $\mathbb{P}_{\mathrm{paths}}$ refers to the distribution over path space $(X_t)_{t\in [0, T]} \in [0, T]\times S$ generated by running CTMC 1, and $H_{u}^{ij} = 1$ represent the transition of state from $x_i$ to $x_j$ at time $u$.
\end{theorem}

\begin{proof}

This result is derived by combining Proposition A.1 and A.2 and rearranging the resulting expression. 

\begin{subequations}
\begin{align}
D_{\mathrm{KL}}\left(\mathbb{P}_{\mathrm{paths}} \| \mathbb{Q}_{\mathrm{paths}}\right) 
& =\mathbb{E}_{\mathbb{P}_{\mathrm{paths}}} \left[ \log \frac{d \mathbb{P}_{\mathrm{paths}}}{d \mathbb{Q}_{\mathrm{paths}}} \right] \\
& =\mathbb{E}_{\mathbb{P}_{\mathrm{paths}}} \left[ \log \eta_t^{-1} \right]\\
& =\mathbb{E}_{\mathbb{P}_{\mathrm{paths}}} \left[
\int_0^t \sum_{i,j=1}^N \kappa_u(i, j) R^1_u(i, j) H^i_u du
- \log \prod_{0 < u \le T} ( 1 + \sum_{i,j=1}^N \kappa_u(i, j)H_u^{ij} )
\right] \\
& =\mathbb{E}_{\mathbb{P}_{\mathrm{paths}}} \left[
\int_0^t \sum_{i,j=1}^N (R^2_u(i, j) - R^1_u(i, j)) H^i_u du
+ \sum_{\substack{H_{u}^{ij} = 1 \\ 0 < u \le T}} (\log R^1_u(i, j) - \log R^2_u(i, j))
\right] \\
& =\mathbb{E}_{\mathbb{P}_{\mathrm{paths}}} \left[
\sum_{\substack{H_{u}^{ij} = 1 \\ 0 < u \le T}} (\log R^1_u(i, j) - \log R^2_u(i, j))
\right].
\end{align}
\end{subequations}

Step (18c) follows from Proposition A.1, while step (18d) is derived by substituting \(\kappa_t(i, j) = \frac{R^2_t(i,j)}{R^1_t(i,j)} - 1\) as defined in Proposition A.2. Additionally, step (18e) utilizes the property of the rate matrix, where \(\sum_{j} R^1_t(i,j) = 0\) for any \(t\) and \(i\).

% Let's now call the distributions generated by each CTMC 1 and 2 at timestep T as $\mathbb{P}_T$ and $\mathbb{Q}_T$, respectively. Since $\mathbb{P}_T$ and $\mathbb{Q}_T$ are marginals of $\mathbb{P}_{\mathrm{paths}}$ and $\mathbb{Q}_{\mathrm{paths}}$ at time $t = T$, by the data processing inequality, the KL-divergence $D_\mathrm{KL}(\mathbb{P}_T||\mathbb{Q}_T)$ is upper-bounded by the KL-divergence $D_\mathrm{KL}(\mathbb{P}_{\mathrm{paths}}||\mathbb{Q}_{\mathrm{paths}})$, which concludes the proof.

Finally, we denote the distributions generated by CTMC 1 and CTMC 2 at time \(T\) as $\mathbb{P}_T$ and $\mathbb{Q}_T$, respectively. Since $\mathbb{P}_T$ and $\mathbb{Q}_T$ are the marginal distributions of $\mathbb{P}_{\mathrm{paths}}$ and $\mathbb{Q}_{\mathrm{paths}}$ at time \(T\), by the data processing inequality, the KL-divergence $D_\mathrm{KL}(\mathbb{P}_T \| \mathbb{Q}_T)$ is upper-bounded by $D_\mathrm{KL}(\mathbb{P}_{\mathrm{paths}} \| \mathbb{Q}_{\mathrm{paths}})$, concluding the proof.
\end{proof}

It is important to note that the KLUB derived here for CTMCs does not fully capture the case of discrete diffusion models, where the reverse rate matrix depends on the state. Nonetheless, we believe the metric derived here provides a useful proxy for estimating the error introduced when using $\tau$-leaping to sample in discrete diffusion models. Finding KLUB for state-dependent CTMCs would be an interesting direction for future work.

\subsection{Proof of \eref{equation:CDEbound}}
\label{appendix:technique0}
Let \( X^i \) refers to the \( i \)-th dimension of \( X \). Define $\pp_{t_{i+1}|t_i}(\cdot|\mathbf{x}_{t_i})\triangleq P_{X_{t_{i+1}}|\mathbf{x}_{t_i}}$
and $\qq_{t_{i+1}|t_i}(\cdot|\mathbf{x}_{t_i})\triangleq P_{X_t^1|\mathbf{x}_s} \otimes\cdots\otimes P_{X_t^d|\mathbf{x}_s}$ We can now derive the following relationship:
\[
\mathcal{E}_{\mathrm{CDE}}(s \shortrightarrow t \,|\, \mathbf{x}_s)
\triangleq D_{KL}\left(
\pp_{t|s} \big\| \qq_{t|s}
\right)
\le D_{KL}\left(
\pp_{[s, t]|s} \big\| \qq_{[s, t]|s}
\right)
= \text{KLUB}_{\mathbf{x}_s}(\mathbb{P}_t \| \mathbb{Q}_t).
\]
where \(\text{KLUB}_{\mathbf{x}_s}\) represents  comparing two continuous-time Markov chains (CTMCs), both starting at the initial point \( \mathbf{x}_s \).

Now we are ready to prove the result:
\begin{align*}
\mathcal{E}_{\mathrm{CDE}}(s \shortrightarrow t) 
&\triangleq \mathbb{E}_{\mathbf{X}_s \sim \mathbb{P}_s} \left[ \mathcal{E}_{\mathrm{CDE}}(s \shortrightarrow t \,|\, \mathbf{X}_s) \right] \\
&\le 
\mathbb{E}_{\mathbf{X}_s \sim \mathbb{P}_s} \left[ 
\textrm{\normalfont{KLUB}}_{\mathbf{X}_s}(\mathbb{P}_t \| \mathbb{Q}_t^{s \shortto t})\right] \\
&=
\mathbb{E}_{\mathbf{X}_s \sim \mathbb{P}_s} \left[ 
\mathbb{E}_{\pp_{t|s}}\left[ 
\sum_{\substack{H_{u}^{ij}=1 \\ 0 < u \le T}}  \log \frac{R^1_u(i,j)}{R^2_u(i,j)}
\right]
\right] \\
&=
\mathbb{E}_{\mathbb{P}_{\mathrm{paths}}}\left[ 
\sum_{\substack{H_{u}^{ij}=1 \\ 0 < u \le T}}  \log \frac{R^1_u(i,j)}{R^2_u(i,j)}
\right] \\
&=
\textrm{\normalfont{KLUB}}_{\mathbb{P}_s}(\mathbb{P}_t \| \mathbb{Q}^{s \shortto t}_t).
\end{align*}

\subsection{Technique 1}
\label{appendix:technique1}

The approximation is made as follows:
\begin{subequations}
\begin{align}
\mathcal{D}_\mathrm{KL}(\mathbb{P}_{path} \| \mathbb{Q}_{path}^{T\shortto 0})
- \mathcal{D}_\mathrm{KL}(\mathbb{P}_{path} \| \mathbb{Q}_{path}^{T\shortto t \shortto 0})
&=
\mathbb{E}_{\mathbb{P}_{path}}\left[
\log \frac{\mathbb{Q}_{path}^{T\shortto t \shortto 0}}{\mathbb{Q}_{path}^{T \shortto 0}}
\right] \label{eq:approx-p_forward} \\
&\approx 
\mathbb{E}_{\mathbb{Q}_{path}^{forward}}\left[
\log \frac{\mathbb{Q}_{path}^{T\shortto t \shortto 0}}{\mathbb{Q}_{path}^{T \shortto 0}}
\right] \label{eq:approx-q_forward} \\
&\approx 
\mathbb{E}_{\mathbb{Q}_{path}^{T\shortto t \shortto 0}}\left[
\log \frac{\mathbb{Q}_{path}^{T\shortto t \shortto 0}}{\mathbb{Q}_{path}^{T \shortto 0}}
\right] \\
&=
\mathcal{D}_\mathrm{KL}(\mathbb{Q}_{path}^{T\shortto t \shortto 0} \| \mathbb{Q}_{path}^{T\shortto 0})
\end{align}
\end{subequations}

Equation \eqref{eq:approx-p_forward} assumes that $\mathbb{P}_{path}\approx \mathbb{Q}_{path}^{forward}$ where $\mathbb{Q}_{path}^{forward}$ refers to the distribution made by forward CTMC. In equation \eqref{eq:approx-q_forward}, we assume that $\mathbb{Q}_{path} \approx \mathbb{Q}_{path}^{T \to t \to 0}$. It is important to note that we can use \eqref{eq:approx-p_forward} as a formula for KLUB computation, as introduced in Algorithm~\ref{alg:klub}. However, the results of JYS sampling schedule optimization show little difference between the two.

Compared to coarser sampling, KLUB computation can be organized as follows:
\begin{subequations}
\begin{align}
\textrm{\normalfont{KLUB}}\big(\mathbb{Q}_0^{T\shortto t \shortto 0} \,\big\|\, \mathbb{Q}_0^{T \shortto 0}\big) 
&= \mathbb{E}_{\mathbb{Q}_{\mathrm{paths}}^{T\shortto t \shortto 0}} \left[
\sum_{i\ne j}\sum_{u=0}^{T} H_{u}^{ij} \log 
\frac{R^{T\shortto t \shortto 0}_u(i,j)}{R^{T \shortto 0}_u(i,j)}
\right] \\
&= \mathbb{E}_{\mathbb{Q}_{\mathrm{paths}}^{T\shortto t \shortto 0}} \left[
\sum_{i\ne j}\sum_{u=0}^{t} H_{u}^{ij} \log 
\frac{R_t(i,j)}{R_T(i,j)}
+ \sum_{i\ne j}\sum_{u=t}^{T} H_{u}^{ij} 
\cancel{\log \frac{R_T(i,j)}{R_T(i,j)}}
\right] 
\label{equation:KLUB-coarse_to_fine-b}
\\
&= \mathbb{E}_{\mathbb{Q}_{\mathrm{paths}}^{T\shortto t \shortto 0}} \left[
\sum_{i\ne j}\log \frac{R_t(i,j)}{R_T(i,j)} 
\sum_{u=0}^{t} H_{u}^{ij}.
\right]
\label{equation:KLUB-coarse_to_fine}
\end{align}
\end{subequations}

In \eref{equation:KLUB-coarse_to_fine-b}, we utilized the fact that under \(\tau\)-leaping, \(R^{T\shortto t \shortto 0}_u(i,j) = R_T(i,j)\) for \(u \in [t, T]\) and \(R^{T\shortto t \shortto 0}_u(i,j) = R_t(i,j)\) for \(u \in [0, t]\). In \eref{equation:KLUB-coarse_to_fine}, the rate matrices are constant over intervals, allowing us to pull \(\log \frac{R_t(i,j)}{R_T(i,j)}\) outside the summation.

\subsection{Technique 2}
\label{appendix:technique2}

Consider the meaning of \(\mathbb{E}\left[\sum_{u=0}^{t} H_{u}^{ij}\right]\); it calculates the average probability of a transition from \(i\) to \(j\) occurring between time \(0\) and \(t\). If we knew \(\partial_u p(x_u=j, x_{u-}=i)\), this could be found by \(\int_0^t \partial_u p(x_u=j, x_{u-}=i) \, du\). However, we do not have access to \(\partial_u p(x_u=j, x_{u-}=i)\).

Fortunately, we do know the conditional transition rate \(\partial_u p(x_u=j \mid x_{u-}=i) = R_t(i,j) \). Let's assume that there are maximally single transition of state in each dimension during the time interval, which is the assumption behind using $\tau$-leaping algorithm for DDMs \citep{campbell2022continuous}. Using this, we can rewrite \eref{equation:KLUB-coarse_to_fine}:

\begin{subequations}
\begin{align}
\mathbb{E}_{\mathbb{Q}_{\mathrm{paths}}^{T\shortto t \shortto 0}}
\left[
\sum_{i\ne j}
\log \frac{R_t(i,j)}{R_T(i,j)} 
\sum_{t < u \leq T} H_{u}^{ij}
\right] 
&= \mathbb{E}_{\mathbb{Q}_{\mathrm{paths}}^{T\shortto t}}
\left[
\mathbb{E}_{\mathbb{Q}_{\mathrm{paths}}^{t \shortto 0}}
\left[
\sum_{i\ne j}
\log \frac{R_t(i,j)}{R_T(i,j)} \sum_{t < u \leq T} H_{u}^{ij} 
\Bigg| X_t = i\right]
\right] \label{equation:KLUB-simple-prev} \\
&\approx \mathbb{E}_{\mathbb{Q}_{\mathrm{paths}}^{T\shortto t}}
\left[
\sum_{X_t\ne j}\log \frac{R_t(X_t,j)}{R_T(X_t,j)} \times R_t(X_t, j) \Delta t  
\right]
\label{equation:KLUB-simple}
\end{align}
\end{subequations}
Equation \eref{equation:KLUB-simple-prev} applies the Law of Total Expectation, and in \eref{equation:KLUB-simple}, we utilize the equation:
\[
\mathbb{E}_{\mathbb{Q}_{\mathrm{paths}}^{t\shortto 0}}\left[\sum_{u=0}^{t} H_{u}^{ij} \Bigg| X_t = i\right] \approx R_t(i,j) \Delta t,
\]
where \(\Delta t = t - 0 = t\). The approximation becomes exact when there are only one transition in each dimension during single interval, and this assumption is true for absorb transition matrix. 

% It is important to note that only one transition in 

% In the original form of KLUB, Equation~\eref{equation:KLUB}, we needed to sample the ratio of rate matrices across \(i\), \(j\), and \(u\), but in Equation~\eref{equation:KLUB-simple}, sampling is only required over \(i\). This is because, given \(i\), we derived the expected value of the rate matrices ratio over \(j\) and \(u\) in closed form. The complete algorithm for KLUB computation combining Techniques 1 and 2 is provided in the Appendix.

\section{Algorithm}

In this section, we present the main algorithm for Jump your steps. 

\subsection{KLUB computation}
\label{appendix:algorithm1}
Please refer to Algorithm~\ref{alg:klub}.
In the case of $k$-Gillespie, we first sampled the corresponding $t$ using the corresponding $p(t|k)$, and then apply Algorithm~\ref{alg:klub}. Here, $p(t|k)$ determined by predefined noise schedule.

\begin{algorithm}[h]
\caption{Computation of $\text{KLUB}(\mathbb{Q}^{s \shortto t \shortto u} \| \mathbb{Q}^{s \shortto u})$}
\label{alg:klub}
\begin{algorithmic}[1] % Adds line numbers for easier reference
\Require 
    \Statex $\theta$: Diffusion model parameters
    \Statex $s, t, u$: Timesteps, with $s > t > u$
    \Statex $p_{data}$: Data distribution
    \Statex $N$: Number of Monte Carlo samples

\Ensure 
    \Statex $\text{KLUB}$: Computed KLUB value

\State Initialize $\text{KLUB}_u \gets 0$ and $\text{KLUB}_d \gets 0$

\For{iteration $= 1$ to $N$}
    \State Sample $X_0 \sim p_{data}$ \Comment{Sample from data distribution}
    \State Sample $X_s \sim q_{s|0}(X_s \mid X_0)$ \Comment{Forward process to $s$}
    \State Sample $X_t \sim p^\theta_{t|s}(X_t \mid X_s)$ \Comment{If we use \eref{eq:approx-p_forward}, \( X_t \sim q_{t|0}(X_t \mid X_0) \).}
    
    \State Set $\Delta t \gets s - t$
    
    \State Update $\text{KLUB}_u \gets \text{KLUB}_u + \sum_j \Delta t R^\theta_t(X_t, j) \log \frac{R^\theta_t(X_t, j)}{R^\theta_T(X_s, j)}$
    \State Increment $\text{KLUB}_d \gets \text{KLUB}_d + 1$
\EndFor 

\State Compute $\text{KLUB} \gets \text{KLUB}_u / \text{KLUB}_d$ \Comment{Final KLUB value}
\State \Return $\text{KLUB}$
\end{algorithmic}
\end{algorithm}

\subsection{Jump your steps}
\label{appendix:algorithm2}
Please refer to Algorithm~\ref{alg:jump-your-steps}.

\begin{algorithm}[h!]
\caption{Jump Your Steps}
\label{alg:jump-your-steps}
\begin{algorithmic}[1]
\Require 
    \Statex $2^K$: Number of function evaluations
    \Statex $T, 0$: Maximum and minimum timesteps
\Ensure 
    \Statex $K \ge 1$: Number of iterations
\State Initialize $\text{Timesteps} \gets (T, 0)$

\For{$k = 1$ to $K$}
    \State Initialize $\text{Timesteps}^* \gets ()$
    
    \For{each pair $(s, u)$ in $\text{Timesteps}[:-1]$ and $\text{Timesteps}[1:]$}
        \State Compute $t \gets \text{GoldenSection}\left(t, \text{KLUB}(\mathbb{Q}^{s \to t \to u} \| \mathbb{Q}^{s \to u})\right)$
        \State Update $\text{Timesteps}^* \gets \text{Timesteps}^* + (t)$
    \EndFor
    
    \State Initialize $\text{Timesteps}^{**} \gets ()$
    
    \For{each pair $(t_i, t_j)$ in $\text{Timesteps}[:-1]$ and $\text{Timesteps}^*[1:]$}
        \State Update $\text{Timesteps}^{**} \gets \text{Timesteps}^{**} + (t_i, t_j)$
    \EndFor
    
    \State Update $\text{Timesteps} \gets \text{Timesteps}^{**}$
\EndFor 

\State \Return $\text{Timesteps}$
\end{algorithmic}
\end{algorithm}

\section{Experiment Details}
\label{appendix:experiment-details}

\paragraph{Golden Section} The golden section search was stopped if the difference between the newly optimized \( t \) and the previous \( t \) was smaller than \( T/2048 \). The maximum number of iterations was set to 32, but usually, the iterations were completed within 8 steps.

\paragraph{CountDown} 
We use SEDD \citep{lou2023discrete} for loss function and the DiT \citep{peebles2023scalable} as a model architecture, the noise schedule followed the log-linear scheme proposed in the SEDD paper. KLUB computation was done with \( \text{num\_samples} = 2048 \), and one golden section search took approximately 4 seconds.

\paragraph{CIFAR10} 
The pretrained model provided by CTMC \citep{campbell2022continuous} was used. 
For CIFAR10, with \( \text{num\_samples} = 1024 \), one golden section search took about 30 seconds.

\paragraph{Monophonic Music} 
The pretrained model provided by CTMC \citep{campbell2022continuous} was used. KLUB computation was performed with 2048 samples, and one golden section search took about 20 seconds. 

\paragraph{Text} 
The pretrained model provided by SEDD \citep{lou2023discrete} was used.
With \( \text{num\_samples} = 256 \), one golden section search took 120 seconds.

\newpage

\section{Additional Results}

\subsection{Ablation study}
\label{appendix:ablation-study}

\begin{wrapfigure}{r}{5.3cm}
\vspace{-2em}
\includegraphics[width=5.3cm]{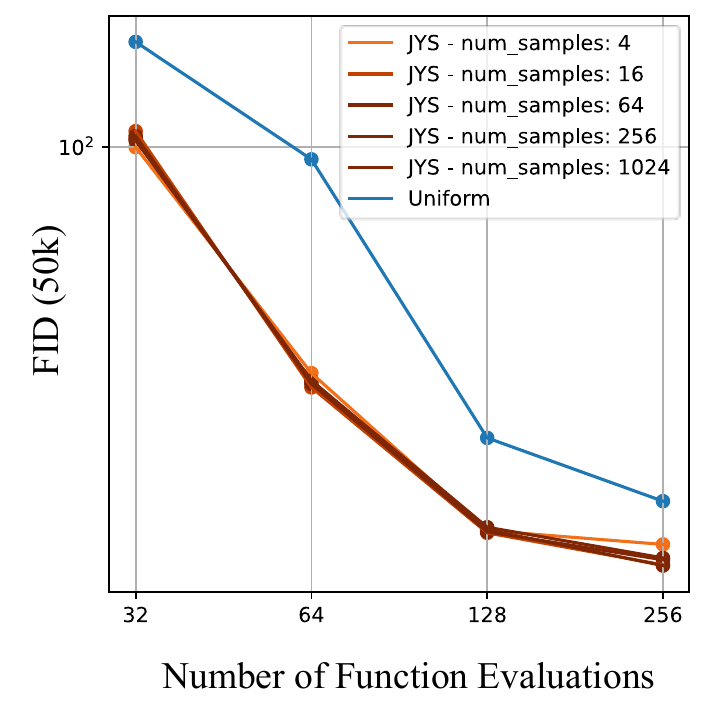}
\caption{
    Performance comparisons on Monophonic music. 
}
\label{figure:ablation-num_samples}
\end{wrapfigure}

\fref{figure:ablation-num_samples} presents the ablation study on the number of samples for KLUB computation. We compare a diverse set of sample sizes, ranging from 4 to 1024. We optimize the JYS sampling schedule for a CTMC trained on CIFAR-10, which has a Gaussian transition matrix. The results show that there was almost no difference. As shown in \fref{figure:klub-distribution}, this seems to be because the KLUB distribution of each sample shows a similar trend with respect to t, so optimization over t does not vary significantly depending on the number of samples. Based on this, it appears that the JYS sampling schedule optimization can be performed with much fewer resources than the computation used in this paper. For example, we can optimize JYS sampling schedule less than 1 minute if we only use 16 samples. 

% Table 7 presents the ablation study of the objective function on
% 642 base resolution. We compare diverse loss variations, including
% the original PaGoDA loss, Lrec + Ladv without CFG, and the loss
% L
% CFG
% dstl + L
% CFG
% adv without reconstruction loss. The vanilla PaGoDA
% loss can only distill ω = 1, and the performance is behind of the
% other models. Using previous guided distillation loss combined with
% adversarial loss, we achieves better performance than the vanilla
% PaGoDA loss. On top of that, if we guide the decoder with the real
% data using the reconstruction loss, we observe additional performance gain. Using CLIP regularization
% further improves the performance, as discussed in Section 4.2.

\subsection{Qualitative results}
\label{appendix:additional-results}

In this subsection, we present qualitative comparison between Jump Your Steps and Uniform sampling schedule under various NFEs. 

% \paragraph{CountDown}

\paragraph{CIFAR10}
\fref{fig:cifar10-qual} shows generated images with various NFE and sampling schedule. All results are generated using Euler $\tau$-leaping sampler.

\begin{figure}[!t]
    \centering
    \includegraphics[width=1\linewidth]{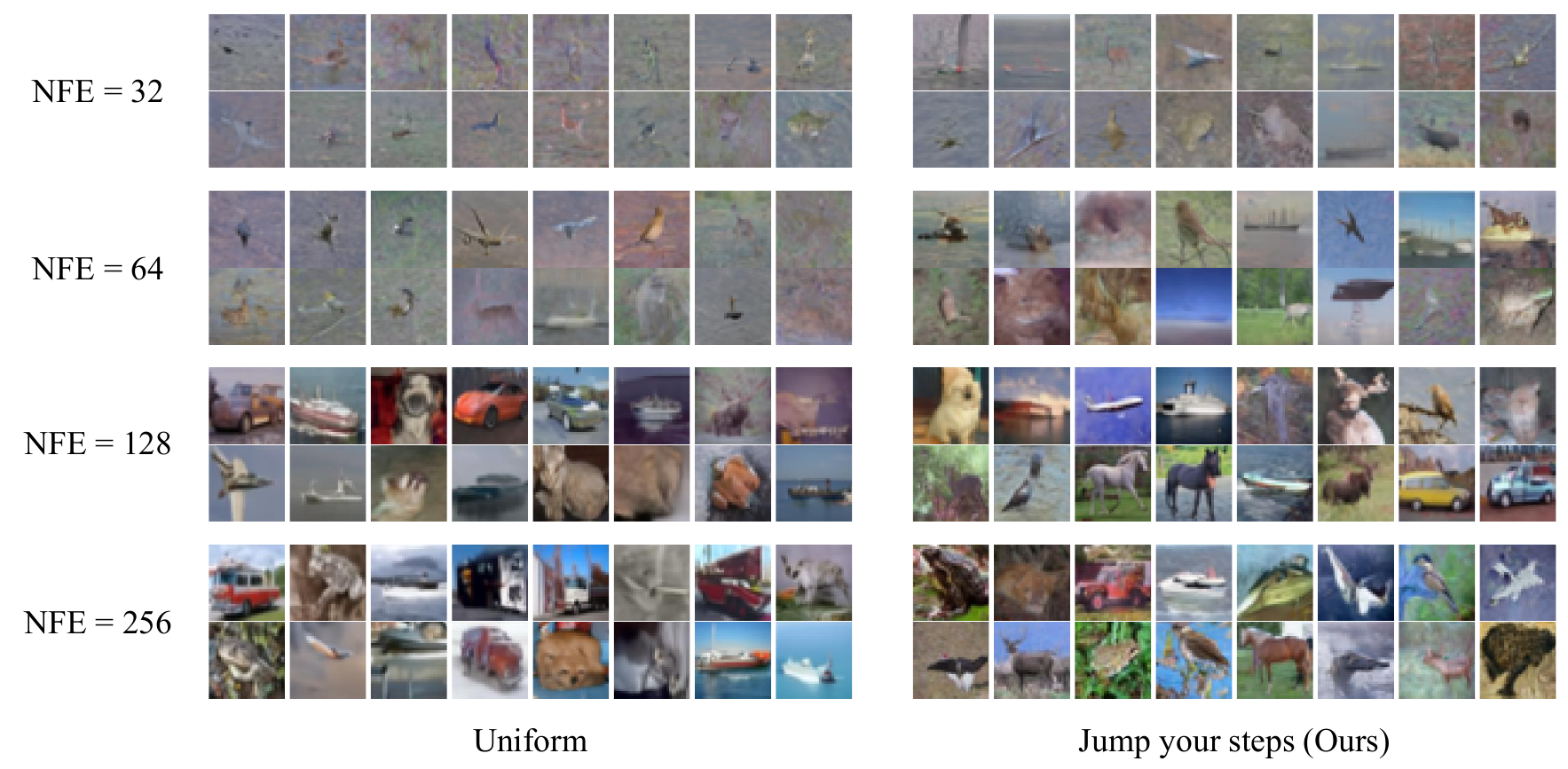}
    \caption{\textbf{CIFAR10 results.} 
    }
    \label{fig:cifar10-qual}
\end{figure}

% \paragraph{Piano}

\paragraph{Text}

Figure~\ref{fig:nfe_64_uniform}, \ref{fig:nfe_64_jys}, \ref{fig:nfe_256_uniform}, \ref{fig:nfe_256_jys} show generated text samples with various NFE and sampling schedule. All results are generated using Euler $\tau$-leaping sampler.

\begin{figure}[!t]
    \centering
    \includegraphics[width=1\linewidth]{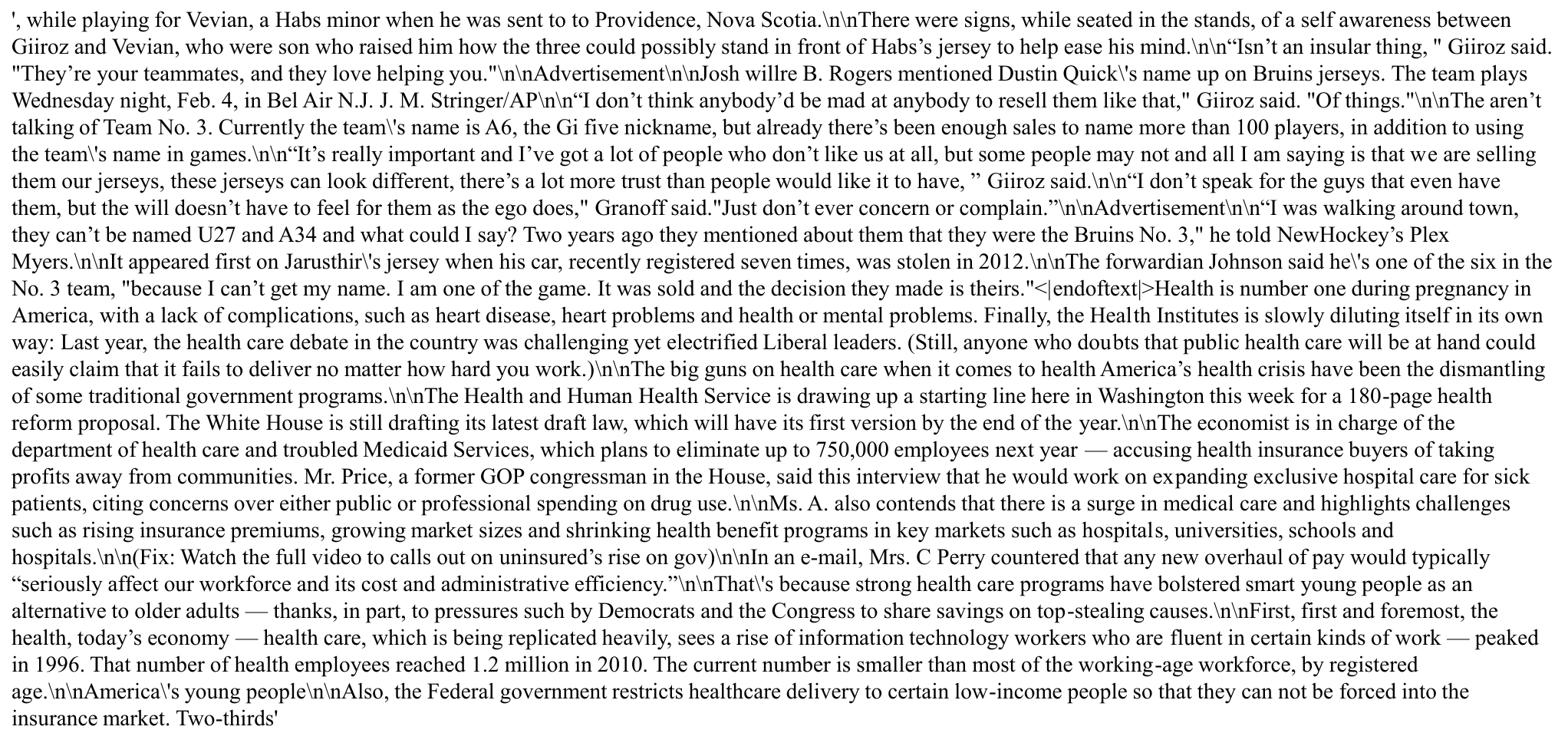}
    \caption{Generated samples using uniform schedule (NFE = 64).
    }
    \label{fig:nfe_64_uniform}
\end{figure}
\begin{figure}[!t]
    \centering
    \includegraphics[width=1\linewidth]{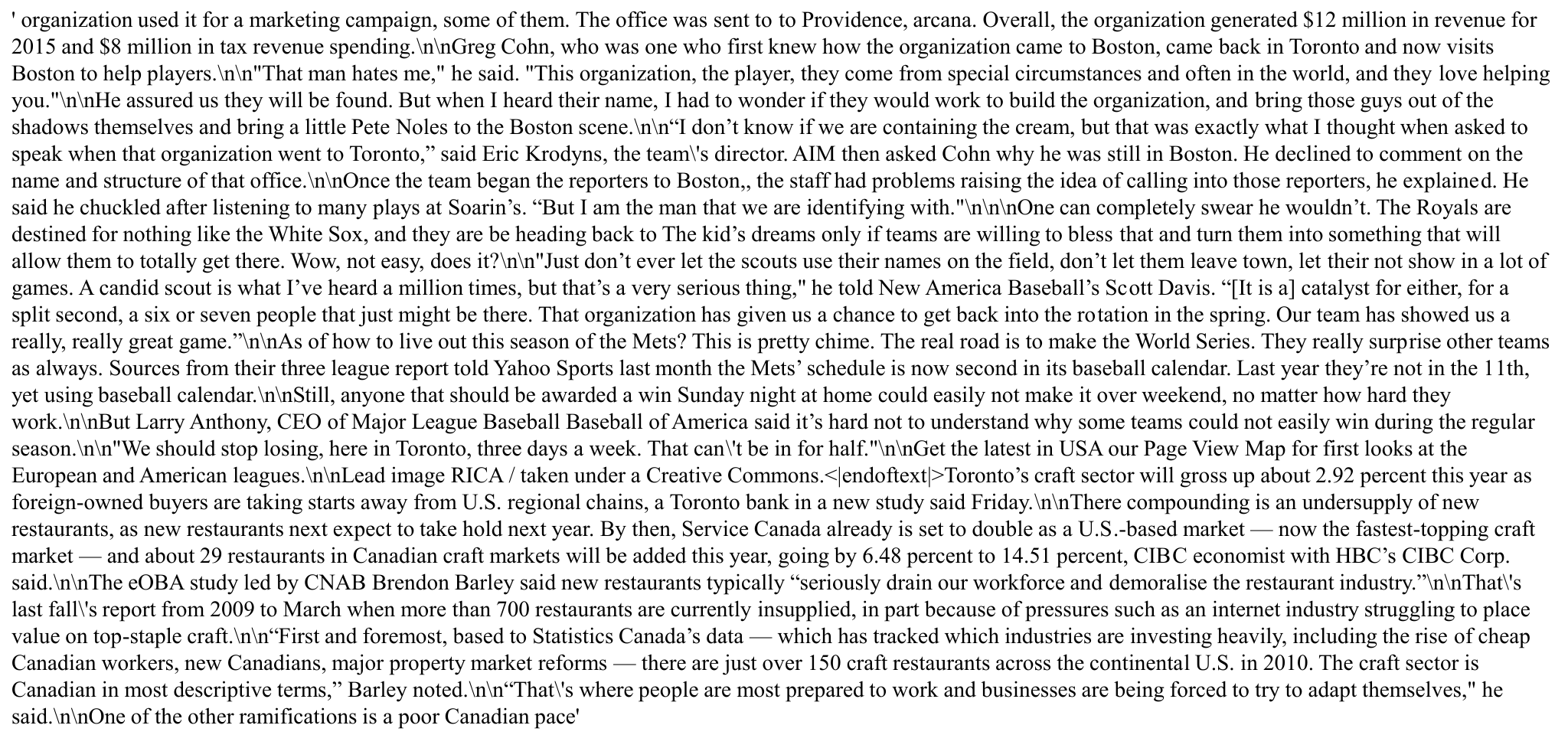}
    \caption{Generated samples using Jump Your Steps schedule (NFE = 64).
    }
    \label{fig:nfe_64_jys}
\end{figure}
\begin{figure}[!t]
    \centering
    \includegraphics[width=1\linewidth]{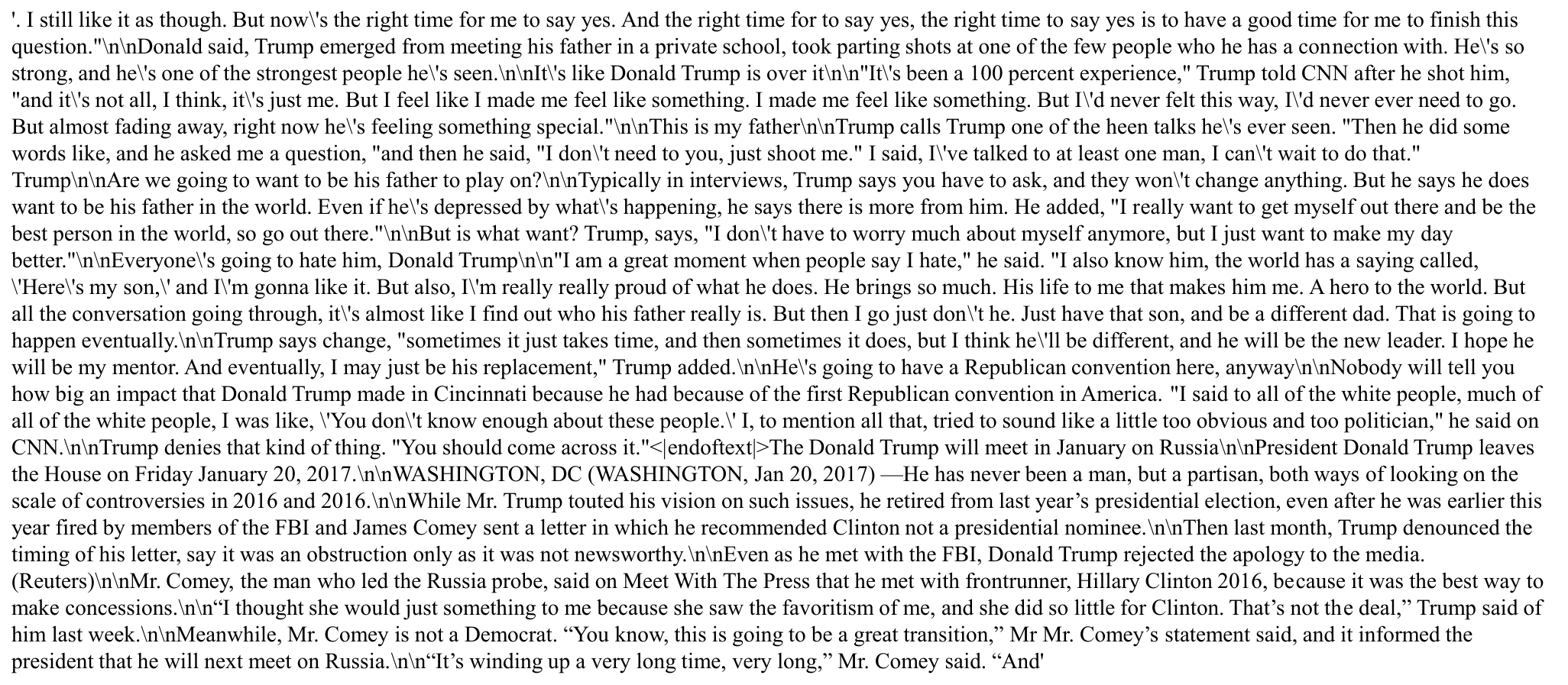}
    \caption{Generated samples using uniform schedule (NFE = 256).
    }
    \label{fig:nfe_256_uniform}
\end{figure}
\begin{figure}[!t]
    \centering
    \includegraphics[width=1\linewidth]{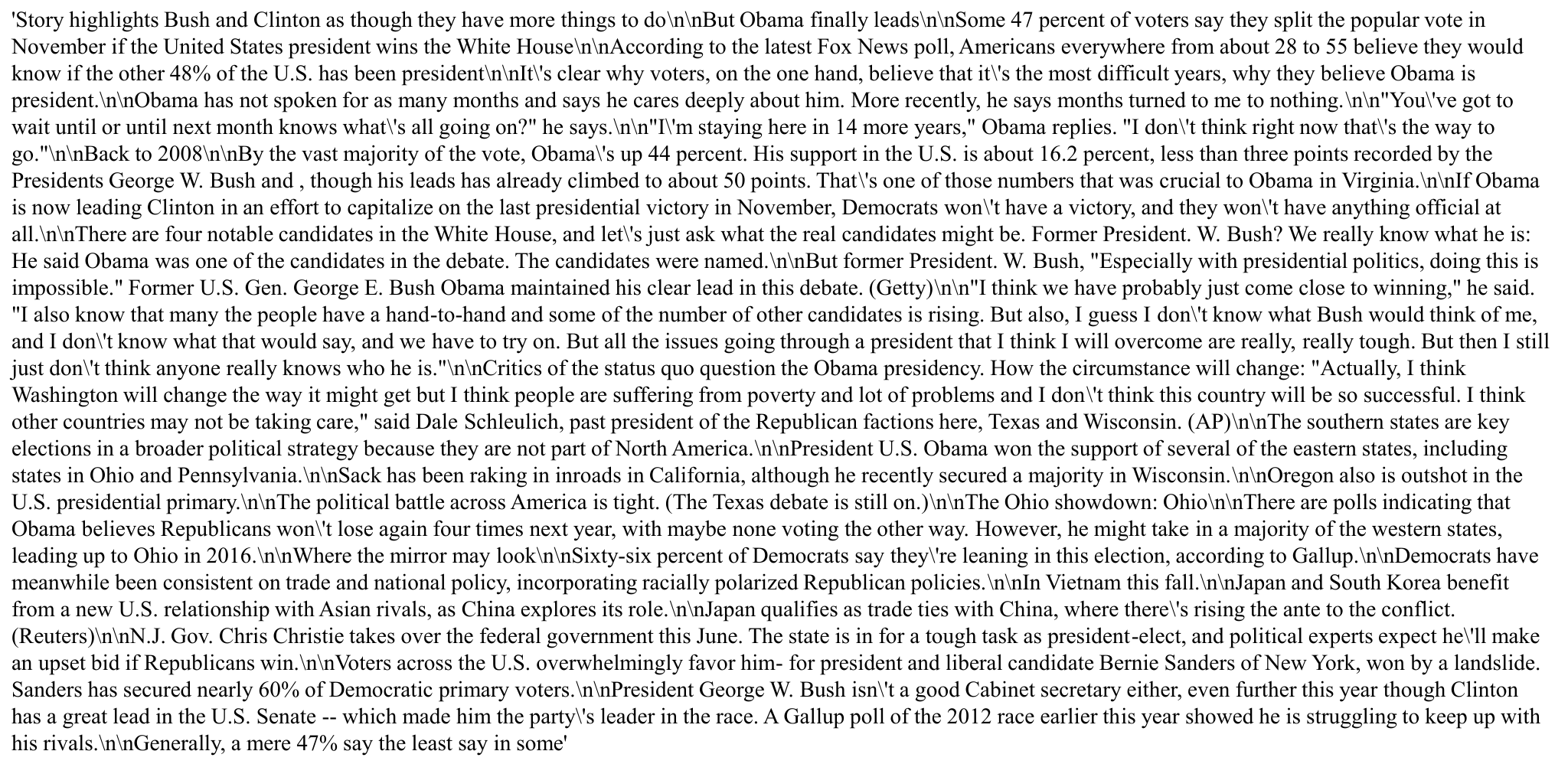}
    \caption{Generated samples using Jump Your Steps schedule (NFE = 256).
    }
    \label{fig:nfe_256_jys}
\end{figure}

\end{document}